%% file: main.tex
\documentclass{article}

\usepackage[main, final]{neurips_2026}

\usepackage[utf8]{inputenc} % allow utf-8 input
\usepackage[T1]{fontenc}    % use 8-bit T1 fonts
\usepackage{hyperref}       % hyperlinks
\usepackage{url}            % simple URL typesetting
\usepackage{booktabs}       % professional-quality tables
\usepackage{amsfonts}       % blackboard math symbols
\usepackage{amsmath}
\usepackage{amsthm}
\usepackage{nicefrac}       % compact symbols for 1/2, etc.
\usepackage{microtype}      % microtypography
\usepackage{xcolor}         % colors
\usepackage{algorithm}
\usepackage{algpseudocode}
\usepackage{geometry}
\usepackage{cleveref}
\usepackage{graphicx}
\usepackage{subcaption}

\newtheorem{proposition}{Proposition}
\newtheorem{definition}{Definition}

\title{From SGD to Muon: Adaptive Optimization via Schatten-p Norms}

\author{%
  Thomas Massena\thanks{Correspondance to thomas.massena@irit.fr}
  \\
  IRIT \& SNCF \\
  \And 
  Corentin Friedrich \\
  IRT Saint Exupery \\
  \And
  Mathieu Serrurier \\
  IRIT \\
}

\begin{document}

\maketitle

\input{sec/00_abstract}
\input{sec/01_introduction}
\input{sec/02_background}
\input{sec/03_rfr_context}

\input{sec/04_estimating_p_star}
\input{sec/05_fractional_update}

\input{sec/06_experiments}

\input{sec/07_conclusion}

\input{sec/acknowledgements}   

\bibliographystyle{plainnat}
\bibliography{biblio}

\appendix

\input{app/01_proofs}
\input{app/02_second_order_p}

\input{app/03_scaling_behaviour}
\input{app/04_fixedp_vs_adaptive}
\input{app/05_p_stability}
\input{app/06_hyperparameters}
\input{app/07_retroactiveness}
\input{app/08_maximal_p}

%%%%%%%%%%%%%%%%%%%%%%%%%%%%%%%%%%%%%%%%%%%%%%%%%%%%%%%%%%%%

% \newpage
% \input{checklist.tex}

\end{document}

%% file: sec/00_abstract.tex
\begin{abstract}
    Modern optimizers, like Muon, impose matrix-wise geometry constraints on their updates. These matrix-wise constraints can be unified under Linear Minimization Oracle (LMO) theory. However, all current methods impose fixed LMO geometries for the update rules, chosen by-design or empirically, which are not necessarily optimal according to the problem's geometry. We introduce a novel efficient criterion for dynamically choosing proxy-optimal LMO update geometries on individual neural network layers. Derived in closed form from gradient and activation statistics using a single-step random feature regression model, our criterion navigates between SGD and Muon-like updates. Moreover, integrating parameter-wise preconditioning allows our framework to recover SGD, Muon, Adam, and MuAdam as specific extrema. To make this adaptive approach scalable, we pair it with efficient computational strategies, achieving only a $\sim 3\%$ runtime overhead on optimized baselines. As a proof of concept, we show that this optimizer beats or remains competitive with the performance of the best performing optimizer between Muon and AdamW across three different training scenarios. Ultimately, this work provides evidence that LMO geometry can be successfully and efficiently adapted from first-order measures, opening a new dimension for robust optimizer design beyond static geometries.
\end{abstract}

%% file: sec/01_introduction.tex
\section{Introduction}
\label{sec:intro}

The training of deep neural networks has traditionally been dominated by coordinate-wise adaptive optimizers like Adam and its variants~\citep{kingma2015adam,loshchilov2019decoupledweightdecayregularization}. By scaling gradient updates element-by-element, these methods are highly effective at managing heterogeneous gradient scales and stochastic noise. However, by treating parameters as flattened vectors, they inherently ignore the structural dependencies of the matrices to optimize, which represent a sizable proportion of trainable parameters in the current transformer architecture dominated landscape~\citep{vaswani2017attention}.
Recently, the Muon optimizer emerged as a powerful alternative that explicitly targets the matrix geometry of 2D parameters~\citep{jordan2024muon}. Instead of parameter-wise scaling, Muon updates weight matrices by using the closest orthogonal matrix to the gradient matrix in Frobenius norm (also known as the matrix polar factor). This operation removes anisotropic scaling and enforces gradient orthogonality, leading to strong empirical performance, even at large scales for large language model pretraining~\citep{liu2025muon}. 
Despite these breakthroughs, the purely geometric approach of Muon exhibits notable vulnerabilities. By treating all singular values equally (effectively applying a flattening operation to the spectrum), Muon lacks the variance tracking and coordinate-wise adaptivity that makes Adam robust. Consequently, Muon can magnify noisy gradient directions on ill-conditioned loss landscapes~\citep{gonon2026insights,li2025normuon,si2025adamuon,dragutinovic2026to}. 

Currently, all existing approaches rely on a heuristically fixed update geometry expressible under the 
Linear Minimization Oracle~\citep{pethick2025training} framework.
In this paper, we propose a generalization that bridges the gap between coordinate-wise adaptivity and spectral geometry. 
Building on the insight that SGDM (SGD with momentum) and Muon are extreme ends of a broader family of updates defined under Schatten-$p$ norms~\citep{schatten1950theory,cesista2025schattenp,peyre2026muondynamicsspectralwasserstein},
we derive a proxy-optimal criterion for the geometry of the update rule, using the framework from~\citet{davis2025spectral}. 
To remain efficient, our method uses only first-order accessible quantities such as activations, gradients and momenta. 

Our contributions are the following:

\begin{itemize}
    \item First, we provide a first-order theoretically principled, single-step, random feature regression model that allows us to pick layer-wise optimal Schatten-$p$ norm constraint descent rules according to the geometry of the problem (\cref{sec:first_order_proxy}).
    \item Then, we study the stability of this $p^*$ proxy across layers throughout training, and observe stable evolutions that undergo sudden phase changes during training. Leveraging this, we find strategies to update layer-wise $p^*$ values in a stable and efficient manner throughout training (\cref{sec:estimator}).
    \item Additionally, we provide a method for the estimation of the optimal $p^*$-dependent update; building on state of the art polar factor estimation methods to minimize runtime overhead and allow for competitive geometry-aware optimization (\cref{sec:pth_root}).
    \item Finally, we validate our method by running a set of diverse experiments, showing that our optimizer matches or surpasses the best performing optimizer between Muon and AdamW in a variety of experimental settings (\cref{sec:experiments}).
\end{itemize}

%% file: sec/02_background.tex
\section{Theoretical Background and Related Work}

First, let us recall the definition of a Schatten-$p$ norm. 

\begin{definition}[Schatten-$p$ Norm]
 Let $X \in \mathbb{R}^{m \times n}$, with singular values $\sigma_i(X)$. For $p \in [1, \infty)$, the Schatten $p$-norm of $X$, denoted $\|X\|_{S_p}$, is defined as the $\ell_p$-norm of its singular values:
\begin{equation}
    \|X\|_{S_p} := \left( \sum_{i=1}^{\min(m,n)} \sigma_i^p(X) \right)^{\frac{1}{p}}
\end{equation}
\end{definition}

In the following section, we provide the theoretical background necessary to our method.

\subsection{Generalizing Muon and SGD as Schatten-p Norm LMOs}
\label{sec:generalization}

Following insights from~\citet{peyre2026muondynamicsspectralwasserstein}, we provide a generalization of both SGDM and Muon's update rules under the framework of Schatten-$p$ norm constrained descent~\citep{bernstein2024old,pethick2025training}. First, let us introduce the framework of Linear Minimization Oracles (LMOs). Given a gradient $G = \nabla_W \mathcal{L}$ with singular value decomposition $U \Sigma V^T$, the LMO aims to solve:

\begin{equation}
    \operatorname{LMO}_\mathcal{D}(G) := \operatorname{argmin}_{\delta W \in \mathcal{D}} \langle G, \delta W \rangle_F
\end{equation}

where $\mathcal{D}$ is the norm ball defined as $\mathcal{D}:=\{x \ | \ \|x\| \leq \rho \}$, with $\rho > 0$ and $\| \cdot \|$ some valid norm. Denoting the Schatten-$p$ norm unit ball as $\mathcal{B}_{S_p} = \{ X \in \mathbb{R}^{m \times n} \ | \ \| X \|_{S_p} \leq 1 \}$, we can express the exact LMO update directions. For any nonzero gradient $G \in \mathbb{R}^{m\times n}$, there exist strictly positive scaling factors $\alpha_{\infty}, \alpha_2 \in \mathbb{R}_{>0}$ such that:

\begin{equation}
\begin{aligned}
   \operatorname{argmin}_{\delta W \in \mathcal{B}_{S_{\infty}}} \langle G, \delta W \rangle_F &= - \alpha_{\infty} U V^T \quad &\text{(Muon)} \\
   \operatorname{argmin}_{\delta W \in \mathcal{B}_{S_2}} \langle G, \delta W \rangle_F &= - \alpha_2 U \Sigma V^T \quad &\text{(SGDM)}
\end{aligned}
\end{equation}

both the analogous Muon and SGDM updates are part of a continuous framework of Schatten constrained norm updates. More generally, we find:

\begin{proposition}
Let $G \in \mathbb{R}^{m \times n}$ have singular value decomposition $G = U\Sigma V^T$, with sorted, positive and nonzero singular values. For any $p \in [1, \infty)$, the solution to the constrained problem, satisfies:
\begin{equation}
    \exists \ \alpha_p \in \mathbb{R}_{>0}, \ s.t. \operatorname{argmin}_{\delta W \in \mathcal{B}_{S_{p+1}}} \langle G, \delta W \rangle_F = -\alpha_p U\Sigma^{1/p}V^T 
\end{equation}
\end{proposition}

Our proof is given in~\cref{app:proof-update}.
In this paper, we propose to adaptively adjust layer-wise LMO update rules between $p=1$ (SGDM) and $p \rightarrow \infty$ (Muon) to improve the training of deep neural networks. 

\subsection{Random Feature Regression}
\label{sec:rfr}

Let us introduce the random feature regression setting from \citet{davis2025spectral}, which the authors use to characterize the regimes that advantage spectral updates over Euclidean updates by comparing their guaranteed loss decrease under optimizer curvature constraints. Consider:

\vspace{-2mm}
\begin{equation}
    \min_{W \in \mathbb{R}^{m \times n}} \mathcal{L}(W) := \frac{1}{2k} \| WA - Y \|_F^2 
\label{eq:rfr}
\end{equation}

where $W \in \mathbb{R}^{m \times n}$ is the $l$-th layer's weight matrix, $A \in \mathbb{R}^{n \times k}$ is the post-activation matrix of the previous layer and $Y \in \mathbb{R}^{m \times k}$ the expected activation matrix of the $l$-th layer. This model is standard and motivated for analytically tractable simplification of neural network behavior  \citep{cho2009kernel,daniely2016toward,rahimi2007random,rahimi2008weighted,rudi2017generalization}.
In this setting, for a weight perturbation $\delta W$, the loss decrease can be approximated via:

\vspace{-2mm}
\begin{equation}
    \mathcal{L}(W+\delta W) - \mathcal{L}(W) = \langle G, \delta W \rangle_F + \frac{1}{2k} \| \delta W . A\|_F^2
\label{eq:loss_decrease}
\end{equation}

\citet{davis2025spectral} show that the Euclidean descent update relies on the operator norm inequality $\|\delta W . A\|_F \le \|\delta W \|_F \|A\|_{2}$, whereas the spectral descent case relies on the Frobenius norm bound $\|\delta W . A\|_F \leq \| \delta W \|_2 \| A \|_F$. 
Ultimately, the authors draw the following conclusion: spectral descent updates outperform gradient descent when the \textit{squared Nuclear-to-Frobenius ratio of the gradient exceeds the} \textit{Stable Rank of the activations} (i.e. $\| G \|_*^2 / \| G \|_F^2 > \| A \|_F^2 / \| A \|_2^2$, with $\| \cdot \|_*$ denoting the nuclear norm). 
Most importantly, the insights from the random feature regression model are shown to transfer to less theoretically well-posed stochastic deep learning pipelines.

\subsection{Related Work}
\label{sec:related}

Structured optimizers such as Shampoo~\citep{gupta2018shampoo}, SOAP~\citep{vyassoap}, and PSGD~\citep{li2017preconditioned,pooladzandi2024curvatureinformedsgdgeneralpurpose} exploit matrix-valued curvature information beyond the diagonal statistics used by Adam and AdamW~\citep{kingma2015adam,loshchilov2019decoupledweightdecayregularization}.
However, removing curvature accumulation can reveal complementary interpretations, since such methods without accumulation can be reduced to steepest descent under matrix norms~\citep{bernstein2024old}. This connection was later formalized through LMO theory~\citep{pethick2025training,veprikov2025preconditionednormsunifiedframework}.

A fixed spectral geometry can, however, lack the coordinate-wise adaptivity needed on ill-conditioned problems~\citep{gonon2026insights,dragutinovic2026to}, motivating hybrid scaling methods~\citep{veprikov2025preconditionednormsunifiedframework,si2025adamuon,li2025normuon,chang2026muoneq} and alternative fixed-$p$ geometries~\citep{qi2026delving,pang2026htmuon}. Recent theory asks when a particular LMO is preferable: \citet{du2026newtonmuonoptimizer} justify the spectral choice under specific right preconditioning and isotropy assumptions, while \citet{pethick2026when} derives favorable geometries analytically in well-defined training regimes. We instead select a closed-form, layer-wise $p^*$ directly from observed training statistics, without prescribing to a particular LMO norm.
Contemporary work by~\citet{shumaylov2026muon} also studies general Schatten geometries and explores optimality criteria in controlled settings, where they show that single-step optimality can be a limitation under certain circumstances; along with this, post-hoc experiments on GPT-2 style pretraining tasks show that the optimal update geometry can be quite unstable. 
We address this in practice with a momentum-aware, temporally smoothed selector together with $p$-dependent coordinate-wise scaling (see~\cref{app:freon}).

%% file: sec/03_rfr_context.tex
\section{Finding the Optimal Geometry}
\label{sec:first_order_proxy}

Here, we extend the random feature regression analysis of \citet{davis2025spectral} to derive a proxy-optimality criterion for the Schatten-$(p+1)$ update geometry.

\subsection{Finding an Optimal Update Geometry under First-Order Moments}
\label{sec:first-order-optimal-p}

Rather than analyzing a moment-free single-step update (which would yield a clean theorem about a moment-less optimizer that is not used in practice) we tie the descent guarantee to a first-order momentum buffer and, optionally, an element-wise second-order preconditioner (c.f. \cref{sec:second-order-moment}). 
We introduce the momentum $M \leftarrow \beta M + (1-\beta) G$, with singular value decomposition $U_M \Sigma_M V_M^T $. The Schatten-$(p+1)$ update, in this case, becomes: $\delta W = - \eta U_M \Sigma_M^{1/p} V_M^T$. Here, our goal is to find $p^*$ such that the loss decrease, independently of scaling, is maximized. We recall:

\begin{equation}
    \mathcal{L}(W + \delta W) - \mathcal{L}(W) = - \eta \langle G, U_M \Sigma_M^{1/p} V_M^T \rangle_F + \frac{\eta^2}{2k} \| U_M \Sigma_M^{1/p} V_M^T \cdot A\|_F^2.
\label{eq:guaranteed-loss-decrease}
\end{equation}

Using the trace cyclic property to reformulate the first and second term of the right hand side, we provide the following result for the optimal Schatten-$(p+1)$ norm geometry according to the spectral statistics of activations, gradients and momenta.

\begin{proposition}[Optimal $p^*$ Proxy]
    For all $p \in [1, \infty)$, the value of $p$ that maximizes the guaranteed descent of~\cref{eq:guaranteed-loss-decrease} under optimal step size the Schatten-$(p+1)$ update can be formulated as:

    \begin{equation}
        p^* = \mathrm{argmax}_{p \in [1, \infty]} \frac{\left(\sum_i C_{i,i} \sigma_{M,i}^{1/p}\right)^2}{\sum_i \sigma_{M,i}^{2/p} B_{i,i}}
        \label{eq:optimal_p_first_order}
    \end{equation}

    with $C_{i,i} = (U_M^T G V_M)_{i,i}$ and $B_{i,i} = \| A^T V_M e_i\|^2 $ defined previously, $G$ the raw unprojected gradient and $A$ the data flowing inside the layer $l$ that uses weights $W$.
\label{prop:p_star_moment_tight}
\end{proposition}

Our complete proof is given in~\cref{app:proof} and the unimodality of the underlying $\operatorname{argmax}$ objective is discussed in~\cref{app:line-search}, along with the stability of the optimal step size in~\cref{app:scaling-factor-update}.

\vspace{-1mm}
\paragraph{Remark.} The random feature regression model is a surrogate tool which is intentionally simplified to provide practical insights in deep learning settings. Our motivation is to provide a simple layer-wise rule that selects $p^*$ from easily accessible quantities, and which applies to every matrix parameter in a network without needing prior knowledge of the forward pass rule tied to that parameter. 

\subsection{Adding Second-Order Moments}
\label{sec:second-order-moment}

In the previous section, we provided a surrogate method for the choice of optimal layer-wise $p^*$ update geometries, allowing for a principled choice between SGDM-like or Muon-like update rules. 
However, this interpolation does not only impact the geometry of the update but also its stability properties.
Indeed, in the high $p^*$ regime, Muon enforces a flattening of the singular value spectra, this provides explicit control of update magnitude and anisotropy at matrix level.
As $p^*$ decreases, this spectral constraint is progressively relaxed, and the training process becomes increasingly sensitive to coordinate-wise scale disparities and stochastic gradient noise.
This motivates the incorporation of an explicit variance control mechanism that provides stability in lower $p^*$ regimes.
We retain this benefit by combining $p^*$ as in~\cref{prop:p_star_moment_tight}, and using an Adam-style accumulator $D \leftarrow \beta_2 D + (1-\beta_2) G \odot G$, where $\odot$ denotes the Hadamard product, to scale the variance of updates. The interface between this preconditioner and the derivations of \cref{sec:first_order_proxy} is examined more in depth in~\cref{app:second_order_p}. Also, as argued in~\citet{qi2026delving}, optimization with different update geometries is more stable in the presence of second-order moments, further justifying the necessity for second order moments in lower $p^*$ regimes.

Here, we rely on the work of~\citet{veprikov2025preconditionednormsunifiedframework} to add second order moment preconditioning to optimization methods that rely on norm constrained updates. 
We propose the following $p$-dependent heuristic to recover Adam updates when $p=1$ and Muon when $p \rightarrow \infty$:

\begin{equation}
    \delta W = -\eta\, D^{\circ -1/(2(p+1))} \odot \mathrm{LMO}_{(p+1)}\!\left(D^{\circ  -1/(2(p+1))} \odot M\right),
\label{eq:lmo_preconditioned}
\end{equation}

where $\mathrm{LMO}_{(p+1)}(X) \propto U \Sigma^{1/p} V^T$ denotes the Schatten-$(p+1)$ norm LMO update.
Here, the choice of the exponent $1/(2(p+1))$ is the unique choice that renders the update invariant to global rescaling of the loss (a property as a key desideratum identified by~\citet{veprikov2025preconditionednormsunifiedframework}) and allows the update rule to interpolate between that of Adam and that of Muon. 
The derivation of this invariance motivated scaling is given in~\cref{app:second_order_p}. Throughout this paper, we will denote the LMO update with no second order moments (recovering SGDM when $p=1$ and Muon when $p \rightarrow \infty$) as SMuon (standing for Schatten-Muon), while the variance controlled LMO update with Adam updates when $p=1$ that uses $p$-dependent moments will be denoted as SMuon (Adam).

%% file: sec/04_estimating_p_star.tex
\section{Estimating Optimal p Values in Practice}
\label{sec:estimator}

\begin{figure}[t!]
    \centering
    \includegraphics[width=0.75\linewidth]{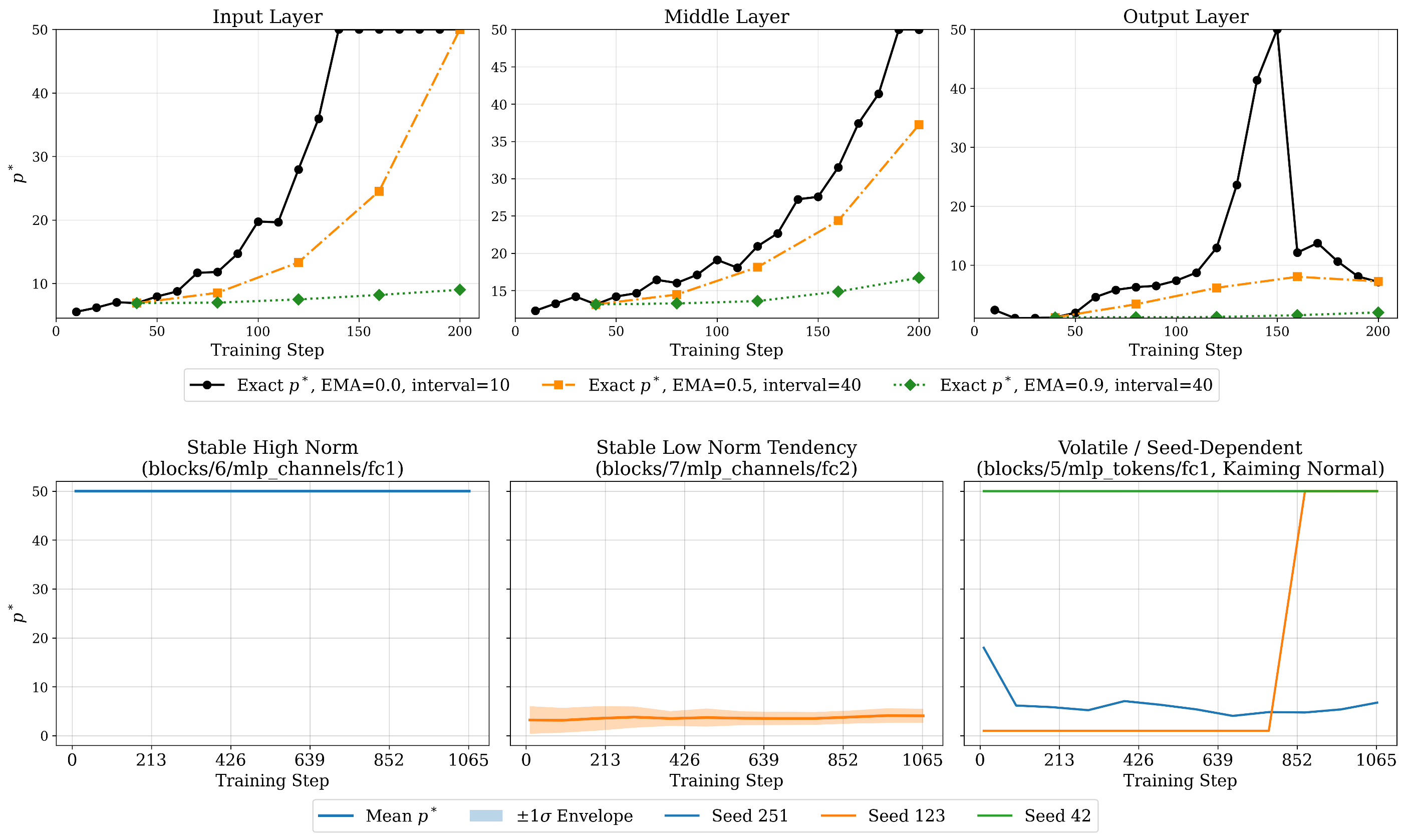}
    \caption{(top) The optimal $p^*$ values on the input, hidden and output layer of an MLP model trained on the MNIST dataset. (bottom) Optimal layer-wise $p^*$ values can exhibit stable or volatile behaviour on an MLP-Mixer trained on the ImageNette dataset.}
    \label{fig:p_stability}
    \vspace{-3mm}
\end{figure}

In this section, we compute optimal $p^*$ values across different settings. Here, we use $p_{\max} = 50$ as the maximum $p^*$ computable value for each layer since $U \Sigma^{1/50} V^T \approx UV^T$ can be considered to be numerically identical to Muon updates for our purposes. Empirically we reveal three important phenomena. Firstly, optimal $p^*$ values seem locally stable, with some sudden phase changes. Secondly, layer-wise $p^*$ values with a varying degree of temporal sampling and optimal $p^*$ computations are compatible with the use of different update frequencies depending on the problem. In order to characterize these phase changes in sparser temporal sampling scenarios, we propose the use of an EMA mechanism with strength $\beta_p$ on the underlying quantities of the $p^*$ computation (i.e., $C_{i,i}$, $\sigma_{M,i}$ and $B_{i,i}$).

\paragraph{Why smooth the spectral quantities rather than $p^*$ ?}
We apply the EMA to $\{C_{i,i}, \sigma_{M,i}, B_{i,i}\}$ rather than directly to $p^*_t$ for two reasons.
First, $p^*$ is the argmax of a nonlinear ratio of these quantities, and argmax does not commute with averaging. Therefore, smoothing the inputs and re-solving for $p^*$ provides a more natural plug-in estimator, while smoothing $p^*_t$ has no
such interpretation. Second, the empirically observed trajectories of $p^*$ (\cref{fig:p_stability}) exhibit abrupt phase transitions between locally stable regimes; smoothing the inputs lets the underlying state evolve continuously
while $p^*$ snaps to whichever regime currently dominates, whereas smoothing $p^*_t$ could drag the estimate
through intermediate values that are themselves suboptimal. The parameter $\beta_p$ then controls a standard
bias--variance tradeoff between minibatch noise and lag behind genuine geometry shifts.

\vspace{-3mm}
\paragraph{Low Frequency $p^*$ Updates (\cref{fig:p_stability}, top)} In a toy experiment, we plot the evolution of the optimal $p^*$ value for an MLP network~\citep{rosenblatt1962principles} trained on 200 steps on the MNIST dataset~\citep{lecun1989handwritten} with learning rate $10^{-2}$ and batch size $4096$. We plot three curves of optimal $p^*$ values for all layer parameters inside the network in~\cref{fig:p_stability} (top): The exact optimal $p^*$ curve, computed every $10$ steps (i.e., the exact, high frequency update), is plotted as a black line. A lower-frequency $p^*$ estimation, computed every $40$ steps, that uses $\beta_p = 0.5$ to smooth out slight local $p^*$ variations, is plotted in orange. Finally, a smoother estimate, with $\beta_p = 0.9$, that also is computed every $40$ steps is displayed in green. The obtained $p^*$ values are given in~\cref{fig:p_stability} (top). In this setting, we notice stable and predictable evolutions of $p^*$ values for each layers, that can undergo sudden phase changes. Also, we recover that using EMA smoothing on the spectral statistics for the computation of $p^*$ can help smooth out the evolution of $p^*$ throughout training; which is rather practical if we are to update the layers with $p^*$ dependent update rules, as we don't want the update rule to fluctuate too randomly during training. Finally, we see that for a rightly chosen EMA strength, here $0.5$, in a rapidly decreasing loss landscape where the optimization geometry shifts rather quickly, the optimal $p^*$ value can be approximated stably and with lower-frequency computations, although it can lag behind the optimal $p^*$ value computed with higher frequency and no EMA. 

\vspace{-3mm}
\paragraph{Stability of $p^*$ Values for Deep Neural Networks} In a larger scale experiment, we train an MLP-Mixer Tiny model with patch size 16 from the $\mathrm{timm}$ library~\citep{rw2019timm} on the ImageNette~\citep{Howard_Imagenette_2019} dataset for 20 epochs. During the training process, we use an adaptive optimizer that computes $p^*$ on the end of every epoch with an EMA of $0.95$. We initialize $p^* = p_{\max}$, effectively starting the training with the Muon update, then, we apply $U \Sigma^{1/p^*_l} V^T$ updates to the matrix parameters for every layer $l$ of proxy optimality parameter $p^*_l$. We record three runs with different random seeds for four different initialization methods, namely $\mathrm{orthogonal}$, $\mathrm{kaiming\_uniform}$, $\mathrm{kaiming\_normal}$, and $\mathrm{truncated\_normal}$~\citep{glorot10a,he2015delving,huorthogonal}. Our observations are threefold.

First, our proxy-optimality $p^*$ criterion can be stable and point towards spectral updates as being optimal as demonstrated in~\cref{fig:p_stability} (bottom - left), where we notice that the first linear layer of the channel-mixing parameter for block 5 of the MLP mixer systematically recommends $p^* \approx p_{\max}$ during the whole training process across all experiments.
Secondly, the proxy $p^*$ criterion can also point towards lower $p^*$ values consistently across all runs and initialization methods, as shown on the second layer of the 7th block of the model (\cref{fig:p_stability}, bottom - middle), where $p^* \ll p_{\max}$ consistently across experiments.
Finally, while the $p^*$ optimal value can be stable in randomized settings for certain layers, some other layers exhibit run-dependent characteristics, even for the same initialization method, as shown in~\cref{fig:p_stability}, (bottom - right) where the same Kaiming Normal initialization method can yield different optimal $p^*$ trajectories for the same parameter matrix.

For the rest of the paper, we will use fixed frequency $p^*$ updates with EMA values in $[0.9, 0.99]$ depending on the neural network. Further ablations on the EMA value and update frequency parameters can be found in~\cref{app:p_stability}, along with $p^*$ visualizations for the NanoGPT runs. 

%% file: sec/05_fractional_update.tex
\section{Efficiently Approximating the Fractional Polar Factor}
\label{sec:pth_root}

Practically, Schatten-$p$ norm updates are built around what we term the ``fractional map direction'', defined by the functional: $\mathcal{F}_p : G \;\longmapsto\; U \Sigma^{1/p} V^T$, when $G = U \Sigma V^T$.

This map smoothly interpolates between vanilla SGDM at $p = 1$ and the Muon orthogonalization step as $p \to \infty$. Computing $\mathcal{F}_p$ exactly via the Singular Value Decomposition (SVD) is prohibitive at scale. Although its asymptotic complexity matches matrix multiplication, SVD is intrinsically sequential, lacks tensor-core acceleration, and severely underutilizes modern GPU and TPU architectures. To resolve this, we require a matrix-multiplication-centric approximation. 
Below, we detail our primary algorithmic contribution: a highly efficient Taylor approximation method built atop any polar factor estimation method. We also briefly discuss alternative polynomial approaches and concurrent methods to contextualize the accuracy-throughput trade-off.

\paragraph{Taylor Approximation via Polar Factor Estimation.}
Our approach leverages the polar factor matrix along with the original update matrix to construct an accurate, hardware-efficient approximation of $\mathcal{F}_p$. This method directly benefits from recent advancements on the Newton-Schulz iterative algorithm and Gram Newton-Schulz methods~\citep{amsel2025polar,ahn2025dion,GramNewtonSchulz,grishina2025accelerating}.

\begin{proposition}[Fractional Update Estimate]
    For any matrix $G \in \mathbb{R}^{m \times n}$ of singular value decomposition $U \Sigma V^T$ and polar factor $P = UV^T$, we can write:

    \begin{equation}
        U \Sigma^{1/p} V^T = \alpha^{1/p}  \sum_{i=0}^\infty \binom{1/p}{i} \left(\frac{1}{\alpha} G P^T - I\right)^i P
    \end{equation}

    for all finite $\alpha \geq \| G \|_2$. 
\label{prop:matrix_approx}
\end{proposition}

Our proof is given in~\cref{app:taylor_approximate}. From any chosen Newton-Schulz implementation, we reliably extract two quantities without resorting to SVD: an upper bound of the spectral norm $\alpha = \|G\|_{S_i}$ (with $i \in \{2, 4\}$ depending on the specific kernel variant~\citep{grishina2025accelerating}), and the polar factor estimate $P \approx U V^T$. 
Importantly, this method only necessitates matrix multiplications applied after the Newton-Schulz routine, which can be made into a nested fused multiply add sequence via Horner's rule. In practice, the truncation order of the binomial expansion of~\cref{prop:matrix_approx} can be treated adaptively, chosen as the minimal degree required for the scalar worst-case error on a logarithmic grid within $[\varepsilon, 1]$ to fall below a predefined user tolerance.
Experimentally, our method places on the dominant, lower-left Pareto frontier (see \cref{fig:matrix_approximation}), providing one to two orders of magnitude higher precision than pure polynomial methods at comparable wall-clock time. Furthermore, its precision degrades more gracefully for large values of $p$ as shown in~\cref{app:maximal_p}.

\begin{figure}[t]
    \centering
    \includegraphics[width=0.6\linewidth]{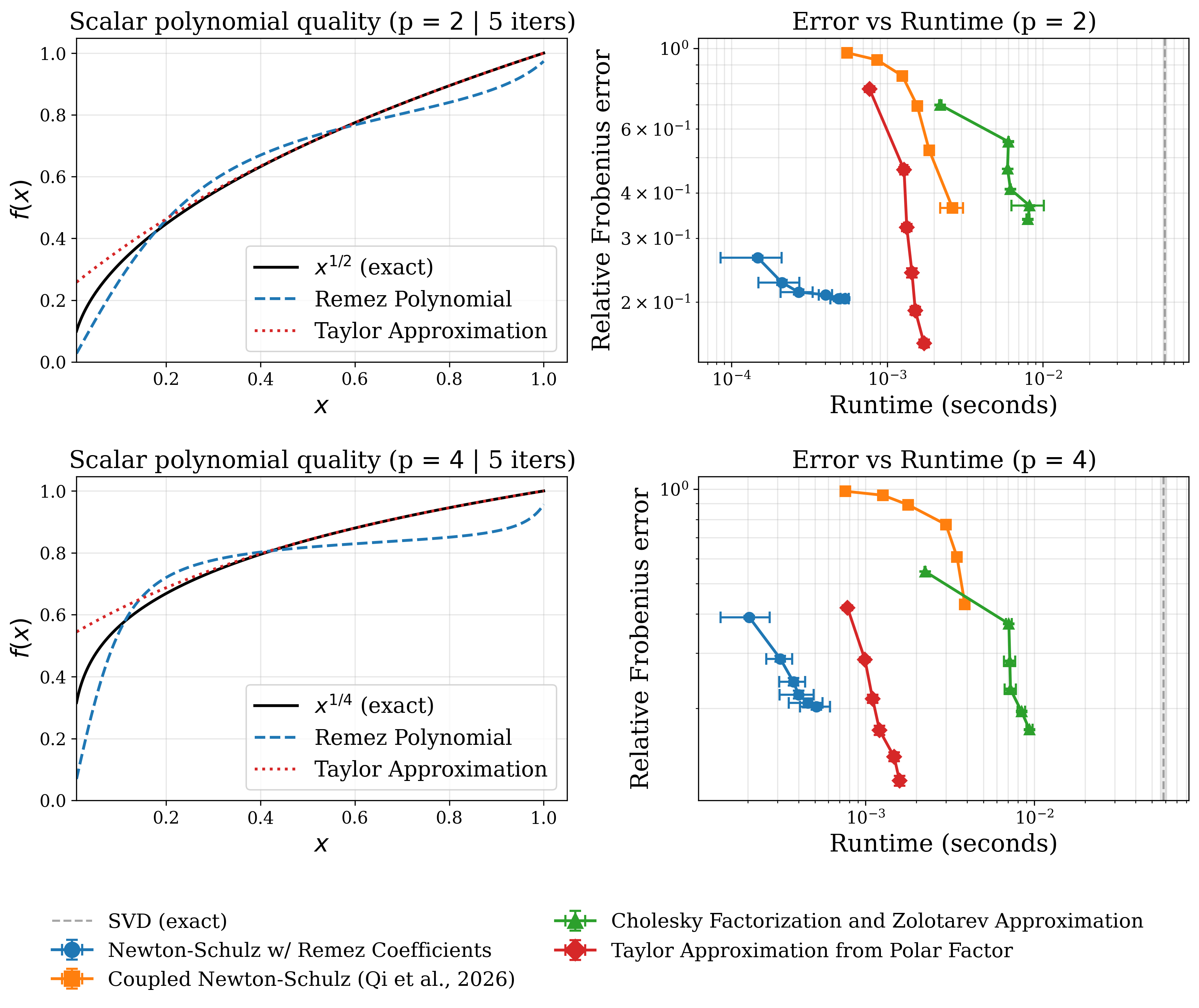}
    \caption{
    Scalar approximation quality (\emph{left}) and wall-clock error-vs-runtime Pareto frontiers (\emph{right}) for $p \in \{2, 4\}$ on $1024 \times 1024$ matrices drawn from a random heavy tailed distribution with condition numbers $[10, 10^3]$ (25 trials, single RTX 4090 GPU). Iterative methods sweep matmul iterative budgets $T \in \{1, \dots, 6\}$; the dashed vertical line marks the exact SVD baseline's runtime. 
    }
    \label{fig:matrix_approximation}
    \vspace{-3mm}
\end{figure}

\paragraph{Concurrent and alternative methods.}
Bypassing the polar factor, one can target $f(x) = x^{1/p}$ directly with Newton-Schulz-style iterations using Remez-fitted odd-degree polynomials. This achieves the fastest absolute runtime (\cref{fig:matrix_approximation}, blue), matching the default Newton-Schulz orthogonalization routine. However, the curvature of $x^{1/p}$ near the origin causes the polynomial approximation to degrade and the $L_\infty$ error degrades sharply for $p \geq 10$ (\cref{app:remez}). \citet{qi2026delving} couple Newton-Schulz recursions to extract $A^{1/2}$ and $A^{-1/2}$ jointly, achieving high accuracy at the cost of doubled working memory and a restriction to dyadic exponents $p = 2^k$, while requiring $k$ nested passes through the coupled Newton-Schulz routine. Zolotarev-based reductions of the fractional factor to $k$ shifted resolvents $(G^T G + s_j I)^{-1}$ accommodate arbitrary $p$ via batched Cholesky, but forfeit the matmul throughput of tensor-core-optimized polynomials, rendering them impractical at training scale (\cref{fig:matrix_approximation}, green).

\paragraph{About Optimization Quality} It is worth noting that perfectly resolving $x^{1/p}$ is not strictly necessary for optimization efficacy. On a standard iteration budget, our Taylor expansion effectively bridges the approximation gap, reaching errors well below the $L_\infty \approx 3.2 \times 10^{-1}$ error of stock Muon coefficients\footnote{As defined in https://github.com/KellerJordan/Muon/blob/master/muon.py} on $[0.02, 1.0]$. Also, importantly, as demonstrated by \citet{gonon2026insights}, moderate spectral error in this regime frequently acts as an implicit regularizer, surprisingly improving performance. Thus, we select the Taylor-Horner method which provides the optimal balance of accelerator-friendly throughput and sufficient theoretical precision.

%% file: sec/06_experiments.tex
\begin{figure}[t!]
    \centering
    \begin{subfigure}[t]{0.47\linewidth}
        \centering
        \includegraphics[width=\linewidth]{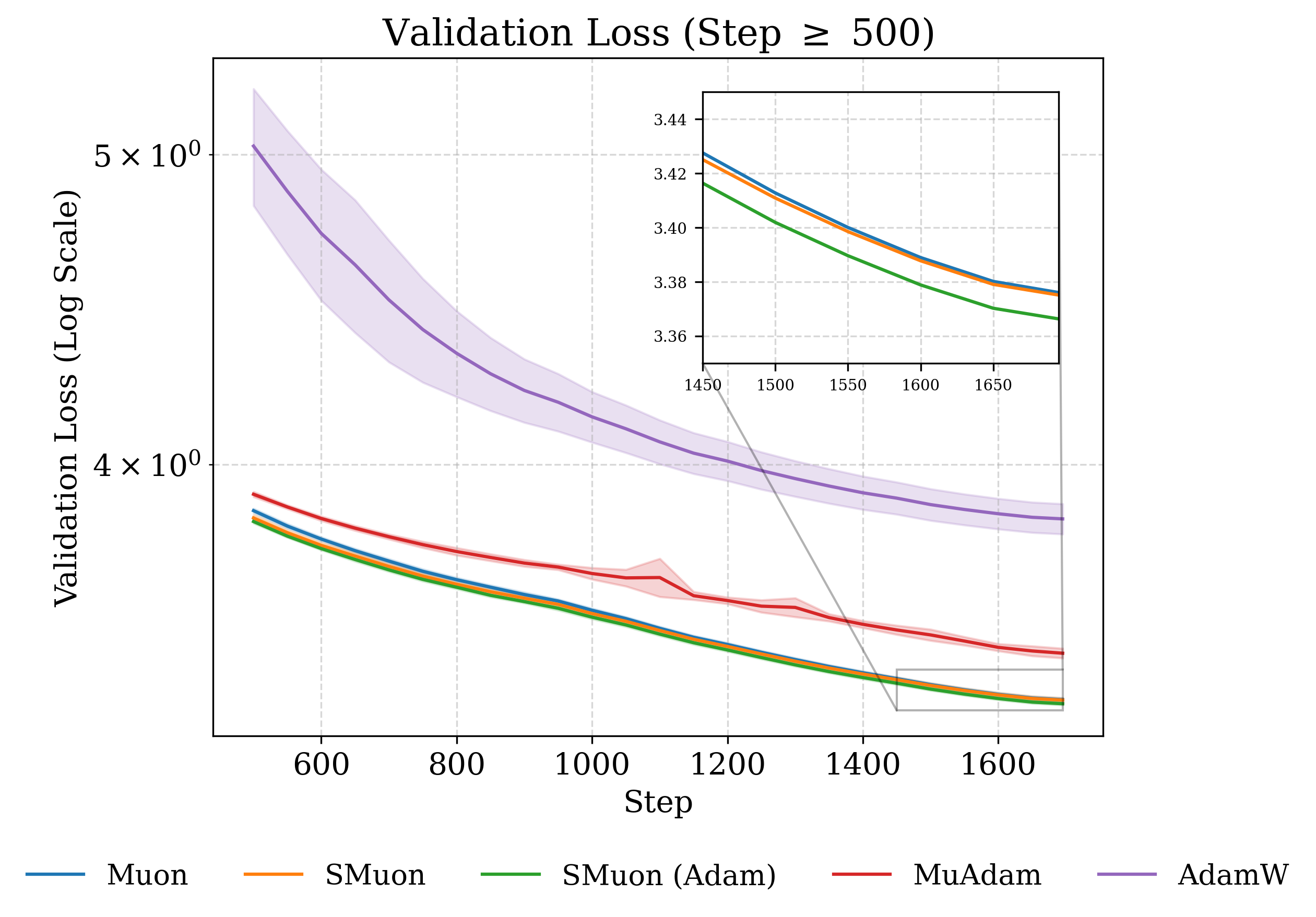}
        \caption{We plot the validation loss value across training steps on the NanoGPT speedrun across 6 different random seeds.  The validation loss is computed on context windows of $262 144$ tokens. 
        }
        \label{fig:nanogpt}
    \end{subfigure}
    \hfill
    \begin{subfigure}[t]{0.49\linewidth}
        \centering
        \includegraphics[width=0.9\linewidth]{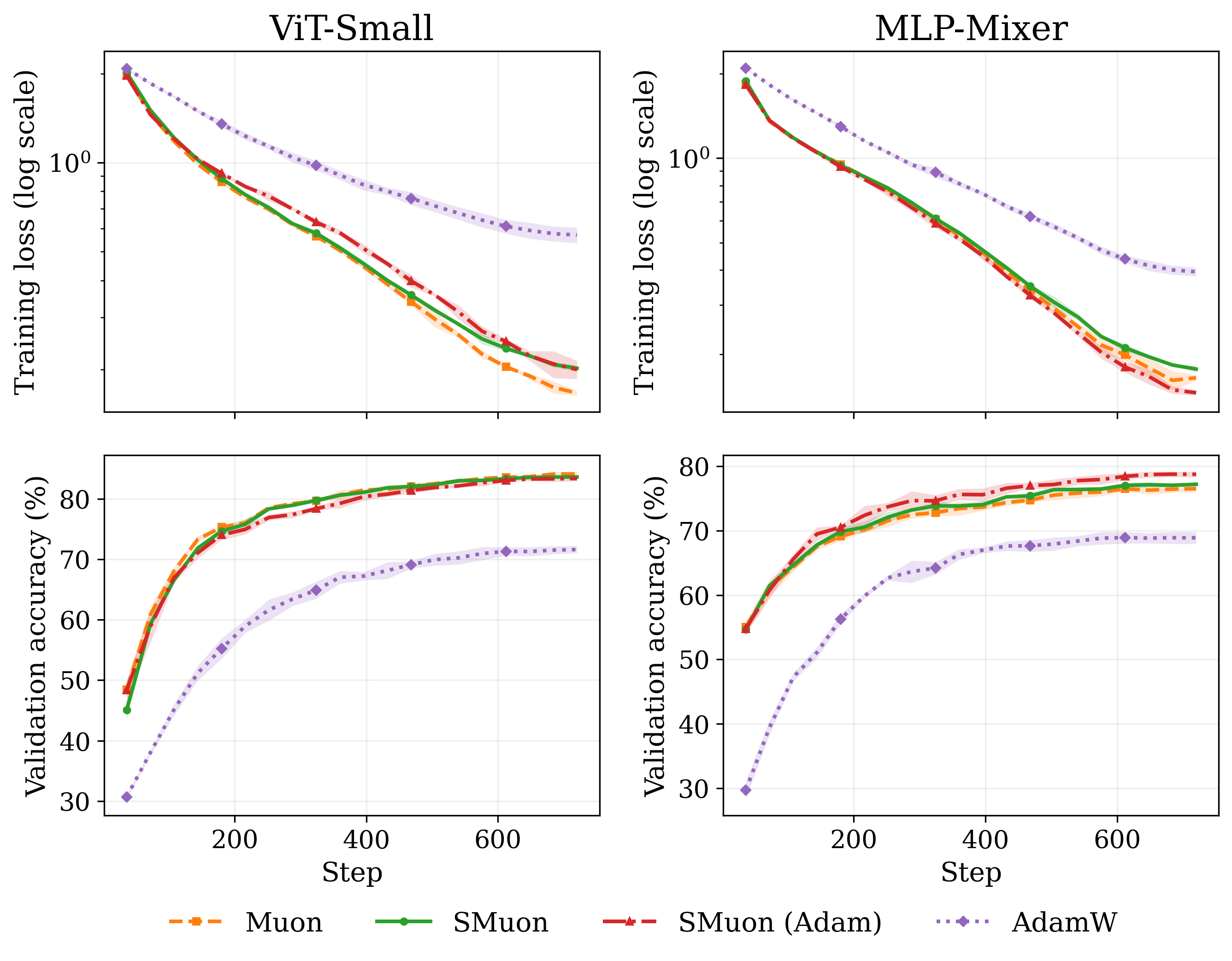}
        \caption{
            We plot the training loss and validation accuracy for each optimizer across 20 epochs on the ImageNette dataset, for both a ViT - Small and an MLP-Mixer - Small model.
        }
        \label{fig:vision}
    \end{subfigure}
    \caption{Training results across language modeling and vision tasks.}
    \label{fig:combined}
    \vspace{-3mm}
\end{figure}

\section{Experimental Observations}
\label{sec:experiments}

We evaluate SMuon across three settings selected to stress different parts of the framework: language model pre-training where Muon excels, a mixed-regime vision task, and low-rank fine-tuning where it appears that a purely spectral geometry update is ill-posed for optimization. 
Each setting reports mean $\pm$ std across seeds, and hyperparameter configurations can be found in~\cref{app:hyperparams}. 

\vspace{1mm}
\textbf{Language model pre-training.} We train a model from the Modded-NanoGPT speedrun \citep{modded_nanogpt_2024} on the FineWeb dataset~\citep{penedo2024fineweb}. To avoid an architecture co-tuned with Muon, we follow the methodology of~\citet{du2026newtonmuonoptimizer}, we adopt an older checkpoint close to record $\#28$\footnote{\href{https://github.com/KellerJordan/modded-nanogpt/blob/9d9dc96/train_gpt.py}{NanoGPT checkpoint URL}}, modified to support our optimizer. All optimizers run with Muon's tuned hyperparameters except MuAdam, for which we sweep the learning rate over six multipliers of Muon's value; AdamW uses the base learning rate found optimal for Muon's auxiliary parameters. SMuon recomputes $p^*$ every 100 steps for a total of $\sim3\%$ overhead when ran on two NVIDIA A100 accelerators. We report the mean validation loss curve across 6 seeds in~\cref{fig:nanogpt}.

The five optimizers we test separate along the two axes of our framework, LMO geometry and element-wise preconditioning. In our experiments, AdamW lags substantially, as it only treats weights coordinate-wise and ignores the matrix geometry that the transformer architecture rewards. MuAdam recovers part of this gap by composing an Adam-style preconditioner with the Schatten-$\infty$ LMO \citep{veprikov2025preconditionednormsunifiedframework}, but performs worse than Muon. 
Muon is our reference, as it is the optimizer of the underlying speedrun record. 
SMuon performs marginally better than Muon, with a vanishing performance gap at the end of training when most $p^*$ values tend to $p_\mathrm{max}$, in accordance with the empirical performance of Muon for pretraining tasks. 
Importantly, SMuon (Adam) improves over Muon on every seed ($3.366 \pm 3.1 \cdot 10^{-3}$ vs $3.376 \pm 4.3 \cdot 10^{-3}$), with a mean per-seed reduction of $9.7 \cdot 10^{-3}$.
The advantage is largest early in training, when layer-wise $p^*$ on the attention matrices drops sharply within the first $\sim$100 steps, and SMuon (Adam) is alone in supplying both the low-$p$ LMO geometry and element-wise variance control on low-$p$ updates. As training progresses and attention $p^*$ climbs back toward $p_\mathrm{max}$, SMuon (Adam) becomes increasingly Muon-like and the gap narrows monotonically. At equal wall-clock time, accounting for the $\sim$3\% overhead, SMuon (Adam) reaches $3.370 \pm 5\cdot 10^{-3}$ versus Muon's $3.376 \pm 4\cdot 10^{-3}$.

\paragraph{Vision.}
We train a ViT-Small and an MLP-Mixer - Small model with a patch size of 16 from $\mathrm{timm}$ on ImageNette for $20$ epochs.
With $p^*$ recomputed once per epoch, we introduce a total of $\sim 0.6\%$ runtime overhead when using an NVIDIA RTX 4090 GPU, we plot the average train loss and validation accuracy across three random seeds for the best performing sweep configuration (see \Cref{fig:vision}). 
On the ViT-Small model, Muon is on par with both adaptive variants and the theoretical proxy shows consistency, as it chooses $p^* \leq 10$ only $5.2\%$ of the time across layers and timesteps, in accordance with the indistinguishable performance gap on this model. 
On the MLP-Mixer - Small model, with more varied $p^*$ trajectories, we observe that SMuon matches the performance of Muon in terms of train loss, while SMuon (Adam) exceeds both. This result consolidates the case for $p$-dependent preconditioning for variance controlled updates.
Importantly, for both these tasks, 20 epochs proves to be too short to reach competitive performance with the AdamW optimizer, demonstrating the better performance of LMO-based counterparts.

\paragraph{Does adaptive geometry itself matter?}
To separate the effect of adapting the Schatten geometry from that of the
$p$-dependent preconditioner, we compare SMuon (Adam) with an otherwise
identical variant in which all matrix parameters use a fixed $p=50$. We also tested lower fixed $p^*$ values which gave worse results.
Across six matched NanoGPT seeds, adaptive SMuon (Adam) reaches every tested
validation-loss threshold $2.4$--$3.0\%$ earlier
, and improves final loss from
$3.3754$ to $3.3663$.
A complementary long-horizon MLP-Mixer experiment compares adaptive
layer-wise $p^*$ against independently tuned fixed values
$p\in\{1.02,3,8\}$; adaptive selection remains competitive with the best
fixed geometry while removing $p$ itself from the hyperparameter search.
Full threshold-wise results, controls, and the fixed-$p$ vision tuning
protocol are reported in~\cref{app:fixed_p}.

\vspace{-2mm}
\paragraph{Low-rank fine-tuning.}
To probe a regime where Muon's full-spectrum spectral geometry is \emph{a priori} suboptimal, we fine-tune a Qwen-2.5-0.5B-Instruct model on the GSM8K dataset with rank-$32$ LoRA adapters over $4$-bit quantized weights~\citep{qwen2,cobbe2021training,hulora,dettmers2023case}, training runs for $3$ epochs. The rank-$32$ structure of the LoRA factors means the polar factor of the momentum matrix has at most $32$ nonzero singular values, and orthogonalization across this degenerate spectrum amplifies the directions corresponding to small singular values, which are a priori noisy. The selector should therefore prefer intermediate $p^*$ on these layers, and indeed we observe an average $p^*$ across layers and training of $2.75$ (SMuon) and $3.10$ (SMuon (Adam)), well below $p_{\max} = 50$. Results across $3$ seeds are reported in~\cref{tab:lora}. Here, we use a $p^*$ computation every epoch with $\beta_p = 0.95$. The picture is relatively clear, with variance controlled optimizers (via orthogonalization or second-order moments), outperforming their counterparts. 
SGDM trails all configurations, while SMuon does slightly better than SGDM but still substantially worse than other configurations; justifying the variance controlled adaptive LMO scheme of~\cref{sec:second-order-moment}.
We also note that Muon performs worse than AdamW and SMuon (Adam) due to its rigid geometry constraints. 
At last, AdamW outperforms all configurations, closely followed by SMuon (Adam), that permits an adaptive choice of LMO in the correct regime (low $p^*$) along with variance controlled updates even when $p^*$ is low.

\begin{table}[h]
\centering
\caption{Optimizer comparison (mean $\pm$ std over 3 seeds).}
\setlength{\tabcolsep}{4pt} % Reduces the horizontal space between columns
\small % Slightly reduces the font size of the table contents
\begin{tabular}{lrrrrr}
\toprule
 & \textbf{SGDM} & \textbf{AdamW} & \textbf{Muon} & \textbf{SMuon} & \textbf{SMuon (Adam)} \\
\midrule
Train loss & $0.736 \pm 7 . 10^{-3}$ & $\mathbf{0.337 \pm 2 . 10^{-3}}$ & $0.387 \pm 2 . 10^{-3}$ & $0.464 \pm 3 . 10^{-3}$ & $0.340 \pm 1 . 10^{-3}$ \\
% Val acc. ($\%$)   & $22.38 \pm 6 . 10^{-3}$ & $\mathbf{29.67 \pm 1 . 10^{-3}}$ & $28.74 \pm 7 . 10^{-3}$ & $28.18 \pm 2 . 10^{-3}$ & $29.37 \pm 7 . 10^{-3}$ \\
Val acc. ($\%$)   & $22.38 \pm 0.6$ & $\mathbf{29.67 \pm 0.1}$ & $28.74 \pm 0.7$ & $28.18 \pm 0.2$ & $29.37 \pm 0.7$ \\
\bottomrule
\end{tabular}
\label{tab:lora}
\vspace{-3mm}
\end{table}

%% file: sec/07_conclusion.tex
\section{Conclusions, Limitations and Future Work}
\label{sec:conclusion}

We introduced a principled framework for adaptive optimization based on Schatten-$p$ norm constrained descent, recovering SGDM and Muon as the endpoints $p{=}1$ and $p{\to}\infty$ of a continuous family of update geometries. Using a single-step random feature regression proxy adapted to first-order momentum, we derived a closed-form, layer-wise selector $p^*$ from quantities a standard optimizer can access. 
This optimization method proved to outperform a heavily-tuned Muon baseline on the Modded-NanoGPT training task, and matched or surpassed Muon on vision tasks.
Also, in voluntarily adversarial settings for Muon, we show that our proxy imposed Adam like updates thus staying very competitive with AdamW on low-rank fine-tuning and completely outperforming Muon in terms of training loss. 
While the absolute gains are modest, the ablation studies reveal consistent and statistically significant improvements across settings, improving optimizer robustness to the training task. These results suggest that the proposed mechanism captures a meaningful optimization effect and motivate further investigation into whether stronger formulations or larger-scale applications can amplify these gains in different settings.

Possible future work could be to explore richer surrogates that account for downstream Jacobians or forward pass behavior~\citep{keigwin2026compositional,li2017preconditioned}.
Also, extending the optimality criterion to a multi-step horizon could be valuable; along with a general analysis of predictability of the scaling behavior of $p^*$ values across model depth and width (see preliminary analysis in~\cref{app:scaling}) as this could allow practitioners to choose predicted-optimal update geometries $p^*$ on larger networks layers, without paying the cost of $p^*$ approximation~\citep{kaplan2020scaling}.

%% file: sec/acknowledgements.tex
\section*{Acknowledgements}

The authors want to thank Louis B\'ethune and Thibaut Boissin for their insightful comments throughout this papers writing process.
Our work has benefited from the AI Cluster ANITI and the research program DEEL\footnote{https://www.irt-saintexupery.com/deel/}. ANITI is
funded by the France 2030 program under the Grant agreement n°ANR-23-IACL-0002. DEEL is an
integrative program of the AI Cluster ANITI, designed and operated jointly with IRT Saint Exupéry, with the financial support from its industrial and academic partners and the France 2030 program under the Grant agreement n°ANR-10-AIRT-01. This project was provided with computing and storage resources by GENCI at IDRIS thanks to the grant 2025-AD011016850 on the supercomputer
Jean Zay’s A100 and H100 partition.

%% file: app/01_proofs.tex
\section{Proofs}
\label{app:proof}

\subsection{The Schatten-p Norm Constrained Update Direction}
\label{app:proof-update}

\begin{proposition}
Let $G \in \mathbb{R}^{m \times n}$ have SVD $G = U\Sigma V^T$, with $\sigma_1 \geq \cdots \geq \sigma_r > 0$ the nonzero singular values. For any $p \in [1, \infty)$, the solution to the constrained problem
\begin{equation}
    \min_{\delta W \in \mathbb{R}^{m \times n}} \; \langle G,\, \delta W \rangle_F \quad \text{s.t.} \quad \|\delta W\|_{S_{p+1}} \leq 1 \tag{$\mathcal{P}$}
\end{equation}
satisfies $\delta W^* \propto U\Sigma^{1/p}V^T$. In the limit $p \to \infty$, the constraint becomes a spectral norm constraint and the solution recovers $\delta W^* \propto \operatorname{PolarFactor}(G) = UV^T$.
\end{proposition}

\begin{proof}

\textbf{Step 1: Decoupling via unitary invariance.}

By von Neumann's trace inequality, for any $A, B \in \mathbb{R}^{m \times n}$:
\begin{equation}
    \langle A, B \rangle_F \geq -\sum_{i=1}^r \sigma_i(A)\, \sigma_i(B),
\end{equation}
with equality if and only if $A$ and $B$ share the same left and right singular vectors, with $B = -U_A \operatorname{diag}(\sigma_i(B)) V_A^T$. Applying this to $A = G$ and $B = \delta W$, the objective $\langle G, \delta W \rangle_F$ is minimised, for any fixed singular values $d_i = \sigma_i(\delta W)$, by taking $\delta W = -UDV^T$ with $D = \operatorname{diag}(d_i)$, yielding:
\begin{equation}
    \langle G, \delta W \rangle_F = -\sum_{i=1}^r \sigma_i\, d_i.
\end{equation}
Since $\|\delta W\|_{S_{p+1}}^{p+1} = \sum_i d_i^{p+1}$ is invariant to the choice of singular vectors, the problem $(\mathcal{P})$ reduces to:
\begin{equation}
    \max_{d_i \geq 0} \; \sum_{i=1}^r \sigma_i\, d_i \quad \text{s.t.} \quad \sum_{i=1}^r d_i^{p+1} \leq 1. \tag{$\mathcal{P}'$}
\end{equation}
The constraint is active at any optimum since all $\sigma_i > 0$, so we work with $\sum_i d_i^{p+1} = 1$.

\textbf{Step 2: Solving the scalar problem via H\"{o}lder's inequality.}

Apply H\"{o}lder's inequality with conjugate exponents $a = \frac{p+1}{p}$ and $b = p+1$, satisfying $\frac{1}{a} + \frac{1}{b} = \frac{p}{p+1} + \frac{1}{p+1} = 1$, to the sequences $x_i = \sigma_i$ and $y_i = d_i$:
\begin{equation}
    \sum_{i=1}^r \sigma_i\, d_i \leq \left(\sum_{i=1}^r \sigma_i^{(p+1)/p}\right)^{p/(p+1)} \left(\sum_{i=1}^r d_i^{p+1}\right)^{1/(p+1)}.
\end{equation}
Since $\sum_i d_i^{p+1} = 1$, this simplifies to:
\begin{equation}
    \sum_{i=1}^r \sigma_i\, d_i \leq \left(\sum_{i=1}^r \sigma_i^{(p+1)/p}\right)^{p/(p+1)}.
\end{equation}
Equality in H\"{o}lder's inequality holds if and only if $x_i^a \propto y_i^b$, i.e. $\sigma_i^{(p+1)/p} \propto d_i^{p+1}$, which gives:
\begin{equation}
    d_i \propto \sigma_i^{1/p}.
\end{equation}
Imposing the normalisation $\sum_i d_i^{p+1} = 1$ with $d_i = c\,\sigma_i^{1/p}$ and solving for $c$:
\begin{equation}
    c^{p+1} \sum_{i=1}^r \sigma_i^{(p+1)/p} = 1 \implies c = \left(\sum_{i=1}^r \sigma_i^{(p+1)/p}\right)^{-1/(p+1)},
\end{equation}
yielding the unique maximiser of $(\mathcal{P}')$:
\begin{equation}
    d_i^* = \frac{\sigma_i^{1/p}}{\left(\displaystyle\sum_{j=1}^r \sigma_j^{(p+1)/p}\right)^{1/(p+1)}}.
\end{equation}

\textbf{Step 3: Reconstruction and absorption into the learning rate.}

The solution to $(\mathcal{P})$ is:
\begin{equation}
    \delta W^* = -U D^* V^T = -\frac{U \Sigma^{1/p} V^T}{\left(\sum_j \sigma_j^{(p+1)/p}\right)^{1/(p+1)}},
\end{equation}
where the denominator is a strictly positive scalar depending only on $G$ and $p$. Absorbing it into the learning rate $\eta > 0$, the descent direction is:
\begin{equation}
    \delta W^* \;\propto\; -U\Sigma^{1/p}V^T.
\end{equation}

\textbf{Step 4: Special cases and the limit $p \to \infty$.}

\begin{itemize}
    \item $p = 1$: $d_i^* \propto \sigma_i$, so $\delta W^* \propto U\Sigma V^T = G$, recovering the \textbf{SGD} update direction.
    \item $p \to \infty$: Since $\lim_{p \to \infty} \|\cdot\|_{S_{p+1}} = \|\cdot\|_{S_\infty}$ \citep{delattre2023efficient}, the constraint in $(\mathcal{P})$ converges to the spectral norm constraint $\|\delta W\|_{S_\infty} \leq 1$. Correspondingly, $\sigma_i^{1/p} \to 1$ for all $i$, so $\delta W^* \propto UV^T = \operatorname{PolarFactor}(G)$, recovering the spectral descent update direction.
\end{itemize}

Analogously, in stochastic settings when using momentum, this translates to SGDM when $p=1$ and Muon when $p \rightarrow \infty$.

\end{proof}

\subsection{Closed Form Optimal p Value From the Random Feature Regression Model}
\label{proof:optimal-p-computation}

We prove the closed-form expression for the optimal Schatten exponent $p^*$ given in~\Cref{prop:p_star_moment_tight} of the main text, restricted to the unpreconditioned case $D = \mathbf{1}$. The general case with $D \neq \mathbf{1}$ is treated in~\cref{app:second_order_p}. We additionally show that the resulting one-dimensional optimization problem in $p$ is well-suited to bounded scalar line search.

\subsubsection{Setting and Notation}

We work in the random feature regression setting of~\cref{sec:rfr}, with loss
\begin{equation}
    \mathcal{L}(W) = \frac{1}{2k}\|W A - Y\|_F^2,
\label{eq:proof-rfr-loss}
\end{equation}
where $W \in \mathbb{R}^{m \times n}$ is a layer's weight matrix, $A \in \mathbb{R}^{n \times k}$ is the post-activation of the previous layer at step $t$, and $Y \in \mathbb{R}^{m \times k}$ is the target. The gradient is $G = \nabla \mathcal{L}(W) = \tfrac{1}{k}(W A - Y) A^T$, and the momentum buffer follows the standard EMA recursion $M \leftarrow \beta M + (1 - \beta) G$.

Let $M = U_M \Sigma_M V_M^T$ be the (thin) singular value decomposition of $M$, with $\Sigma_M = \mathrm{diag}(\sigma_{M,1}, \ldots, \sigma_{M,k})$ and $k = \min(m, n_\text{out})$. The Schatten-$p$ momentum descent update considered in this section is
\begin{equation}
    \delta W = -\eta\, U_M \Sigma_M^{1/p} V_M^T, \qquad p \in [1, \infty],
\label{eq:proof-update}
\end{equation}
which recovers SGDM at $p = 1$ and Muon as $p \to \infty$.

\subsubsection{Exact Loss Decomposition}

The squared loss~\eqref{eq:proof-rfr-loss} is a convex quadratic in $W$, and the loss change from a perturbation $\delta W$ admits the exact decomposition
\begin{equation}
    \mathcal{L}(W + \delta W) - \mathcal{L}(W) = \langle G, \delta W\rangle_F + \frac{1}{2k}\|\delta W \cdot A\|_F^2.
\label{eq:proof-loss-decomp}
\end{equation}
This is an equality, the cross-term $\langle G, \delta W\rangle_F$ captures the first-order descent and the quadratic term $\tfrac{1}{2k}\|\delta W A\|_F^2$ captures the curvature of $\mathcal{L}$ along the update direction, with no higher-order remainder.

\subsubsection{Simplification of the First-Order Term}

Substituting~\eqref{eq:proof-update} into the first-order term and expanding the inner product in the singular basis,
\begin{align}
    \langle G, \delta W\rangle_F
    &= -\eta \,\langle G, U_M \Sigma_M^{1/p} V_M^T\rangle_F \nonumber\\
    &= -\eta \,\mathrm{Tr}\bigl(\Sigma_M^{1/p}\, V_M^T G^T U_M\bigr) \nonumber\\
    &= -\eta \,\mathrm{Tr}\bigl(\Sigma_M^{1/p}\, C^T\bigr) \nonumber\\
    &= -\eta \sum_{i=1}^k \sigma_{M,i}^{1/p}\, C_{i,i},
\label{eq:proof-first-order}
\end{align}
where we have introduced the \emph{alignment matrix}
\begin{equation}
    C := U_M^T G V_M \in \mathbb{R}^{k \times k},
\label{eq:proof-C-def}
\end{equation}
and $C_{i,i}$ are its diagonal entries. The off-diagonal entries of $C$ do not appear in~\eqref{eq:proof-first-order} because the trace selects only the diagonal of the product $\Sigma_M^{1/p} C^T$.

\paragraph{Interpretation.} $C_{i,i}$ measures the alignment between the current gradient $G$ and the $i$-th singular component of the momentum-accumulated gradient $M$. When $G$ is well aligned with $M$ (e.g., during stable phases), the leading $C_{i,i}$ are large and positive; when $G$ deviates from the momentum direction, the $C_{i,i}$ shrink or change sign. The sign of $C_{i,i}$ may flip across the spectrum, reflecting partial misalignment between gradient and momentum.

\subsubsection{Simplification of the Curvature Term}

Substituting~\eqref{eq:proof-update} into the curvature term and using the orthonormality of $U_M$,
\begin{align}
    \|\delta W \cdot A\|_F^2
    &= \eta^2\, \|U_M \Sigma_M^{1/p} V_M^T A\|_F^2 \nonumber\\
    &= \eta^2\, \mathrm{Tr}\bigl(A^T V_M \Sigma_M^{2/p} V_M^T A\bigr) \nonumber\\
    &= \eta^2\, \mathrm{Tr}\bigl(\Sigma_M^{2/p}\, V_M^T A A^T V_M\bigr) \nonumber\\
    &= \eta^2 \sum_{i=1}^k \sigma_{M,i}^{2/p}\, B_{i,i},
\label{eq:proof-curvature}
\end{align}
where
\begin{equation}
    B_{i,i} := (V_M^T A A^T V_M)_{i,i} = \|A^T V_M e_i\|^2.
\label{eq:proof-B-def}
\end{equation}
Each $B_{i,i} \geq 0$ measures the energy of the $i$-th right-singular direction of $M$ when projected onto the activation $A$. The orthonormality of $U_M$ is what eliminates cross-singular-value couplings: $\|U_M X\|_F^2 = \|X\|_F^2$ for any $X$, so the $U_M$ factors cancel cleanly. This is a special property of $D = \mathbf{1}$ and is the structural reason why the closed-form $p^*$ exists in this case.

\subsubsection{Closed-Form Optimal Step Size}

Combining~\eqref{eq:proof-loss-decomp},~\eqref{eq:proof-first-order}, and~\eqref{eq:proof-curvature}, the loss decrease is
\begin{equation}
    \mathcal{L}(W) - \mathcal{L}(W + \delta W) = \eta\, N(p) - \frac{\eta^2}{2k}\, D(p),
\label{eq:proof-quadratic-eta}
\end{equation}
where we have defined
\begin{equation}
    N(p) := \sum_{i=1}^k \sigma_{M,i}^{1/p}\, C_{i,i}, \qquad D(p) := \sum_{i=1}^k \sigma_{M,i}^{2/p}\, B_{i,i}.
\label{eq:proof-N-D}
\end{equation}
The right-hand side of~\eqref{eq:proof-quadratic-eta} is concave quadratic in $\eta$, with maximum at the co-optimal step size
\begin{equation}
    \eta^*(p) = \frac{k\, N(p)}{D(p)},
\label{eq:proof-eta-star}
\end{equation}
attained whenever $N(p) > 0$ (otherwise the maximum is at $\eta = 0$, i.e., the update direction is not a descent direction and the quadratic model recommends not moving). Substituting $\eta^*(p)$ back into~\eqref{eq:proof-quadratic-eta} yields the maximum loss decrease at exponent $p$:
\begin{equation}
    \bigl[\mathcal{L}(W) - \mathcal{L}(W + \delta W)\bigr]_{\eta = \eta^*(p)} = \frac{k}{2}\, \frac{N(p)^2}{D(p)}.
\label{eq:proof-max-decrease}
\end{equation}

\paragraph{Remark} In practice, all $\eta^*(p)$ values we computed on deep learning tasks turned out to be positive \textit{across all vision experiments where they where monitored}. However, in the case where $\eta^*(p) \leq 0$, a fallback heuristic could be to leave the $p^*$ value unchanged for that particular step.

\subsubsection{Closed-Form Optimal p}

The factor $k/2$ in~\Cref{eq:proof-max-decrease} is independent of $p$, so the value of $p$ that maximizes the loss decrease under co-optimal step size is
\begin{equation}
    \boxed{\;p^* = \operatorname{argmax}_{p \in [1, \infty]}\; \frac{N(p)^2}{D(p)} = \operatorname{argmax}_{p \in [1, \infty]}\; \frac{\bigl(\sum_i \sigma_{M,i}^{1/p}\, C_{i,i}\bigr)^2}{\sum_i \sigma_{M,i}^{2/p}\, B_{i,i}}\;}
\label{eq:proof-p-star}
\end{equation}
which is the expression stated in~\Cref{prop:p_star_moment_tight}. Note that the optimal step size $\eta^*(p)$ is itself $p$-dependent, but its value does not enter the criterion for $p^*$: the $\eta$-degree-of-freedom has been integrated out by the closed-form maximization in~\eqref{eq:proof-eta-star}, leaving $p^*$ as a function of the spectrum and the alignment/curvature coefficients only. 
% The criterion can equivalently be read as a \emph{correlation-to-curvature ratio}: $N(p)$ measures the $p$-weighted correlation between gradient and momentum-singular structure, $D(p)$ measures the corresponding curvature, and $p^\star$ chooses the spectral geometry that maximizes the squared ratio.
\hfill$\square$

\subsubsection{Suitability of Bounded Scalar Line Search}
\label{app:line-search}

Equation~\eqref{eq:proof-p-star} is a one-dimensional optimization over $p \in [1, \infty]$. We argue that it is well-suited to a bounded scalar line search algorithm such as Brent's method~\citep{brent2013algorithms}, with three structural properties supporting this claim.

\paragraph{Bounded effective domain.} Although the criterion is defined formally on $[1, \infty]$, the dependence of $N(p)$ and $D(p)$ on $p$ is through $\sigma_{M,i}^{1/p}$ and $\sigma_{M,i}^{2/p}$. As $p \to \infty$, both $\sigma_{M,i}^{1/p} \to 1$ and $\sigma_{M,i}^{2/p} \to 1$, so the criterion saturates to a finite limit, assuming $\sigma_{M,i} > 0$:
\begin{equation}
    \lim_{p \to \infty}\; \frac{N(p)^2}{D(p)} = \frac{\bigl(\sum_i C_{i,i}\bigr)^2}{\sum_i B_{i,i}}.
\label{eq:proof-p-infty-limit}
\end{equation}
The criterion becomes essentially flat once $p$ is large enough that all $\sigma_{M,i}^{1/p}$ are close to unity. We therefore restrict the search to a bounded interval $[p_\text{min}, p_\text{max}]$ chosen above this saturation regime; in our experiments $p_\text{min} = 1.02$ (slightly above 1 to avoid the degenerate $\sigma_{M,i}^1$ summation) and $p_\text{max} = 50$ are sufficient.

\paragraph{Smoothness.} Both $N(p)$ and $D(p)$ are infinitely differentiable functions of $p$ on $(1, \infty)$, since each summand $\sigma_{M,i}^{q/p}$ is smooth in $p$ for $\sigma_{M,i} > 0$ and $q \in \{1, 2\}$. The criterion $J(p) := N(p)^2 / D(p)$ is therefore $C^\infty$ wherever $D(p) > 0$, which holds whenever there exists at least one index $i$ with both $\sigma_{M,i} > 0$ and $B_{i,i} > 0$ which represents a generic condition for non-degenerate momentum and activation matrices.

\paragraph{Mild non-convexity, well-behaved in practice.} $J(p)$ is not unimodal in full generality: the squared correlation $N(p)^2$ can have local minima where $C_{i,i}$ entries cancel, producing transient zeros of $N(p)$ separating regions of opposite sign. In stable training phases, however, the dominant $C_{i,i}$ entries are positive and consistently signed the gradient and momentum agree on their leading components and we observe smooth, single-maximum landscapes throughout these phases. The rare cases in which $J(p)$ becomes non-unimodal coincide with phase transitions;
% where the criterion's underlying assumptions are themselves momentarily strained
the temporal EMA on the alignment and curvature coefficients (\cref{sec:estimator}) absorbs these transients without affecting the surrounding $p^*$ trajectory. Brent's method, converges rapidly on the smooth landscapes that dominate training and degrades gracefully, returning a local maximum within the bracket, on the rare non-unimodal steps. We have not observed pathological behavior of the line search in any of our experiments.

\paragraph{Cost per evaluation.} Once the SVD of $M$ and the alignment/curvature coefficients $\{\sigma_{M,i}, C_{i,i}, B_{i,i}\}_{i=1}^k$ are precomputed, each evaluation of $J(p)$ costs $O(k)$ floating-point operations: two $k$-fold sums and a division. The full line search therefore costs $O(k \cdot n_\text{eval})$ on top of the one-time SVD cost, which is anecdotic compared to the SVD itself. This is the regime in which the bounded scalar minimization is negligible, and is one of the reasons why the unpreconditioned criterion of~\eqref{eq:proof-p-star} is computationally attractive. This property is lost when general second-order moments are introduced (see \cref{app:second_order_p}).

\subsection{Considering the Scale of Updates}
\label{app:scaling-factor-update}

Throughout this study, we choose the optimal update geometry according to a random feature regression surrogate that assumes optimal step sizes. Importantly, this step size can be computed in practice. However, we found this factor to be variable in the $[10^{-2}, 10^4]$ range across model layers. Importantly, we found that using optimal step size multipliers according to the random feature regression model could degrade training as some $\eta^*_t$ values could become higher than $\eta^*_{t-1}$ after a $p^*_{t-1} \rightarrow p^*_t$ update, thus creating loss spikes. 
Furthermore, scaling each update with its optimal step size according to its LMO formulation requires the computation of the Schatten-$(p+1)$ norm of the update, which is costly in practice.
Finally, as previously shown in~\citet{qi2026delving} for fixed $p$ variants, second order preconditioning seems to importantly mitigate this phenomenon, thus explaining the poor performance of SMuon with no second order preconditioning in some scenarios.

\section{Proof of Taylor Approximation of $U \Sigma^{1/p} V^T$ in~\cref{prop:matrix_approx}}
\label{app:taylor_approximate}

\begin{proposition}[Fractional Update Estimate]
    For any matrix $G \in \mathbb{R}^{m \times n}$ of singular value decomposition $U \Sigma V^T$ and polar factor $P = UV^T$, we can write:

    \begin{equation}
        U \Sigma^{1/p} V^T = \alpha^{1/p}  \sum_{i=0}^\infty \binom{1/p}{i} \left(\frac{1}{\alpha} G P^T - I\right)^i P
    \end{equation}

    for all finite $\alpha \geq \| G \|_2$. 
    % with $M = \sum_{i=0}^\infty \binom{1/p}{i} (\frac{1}{\alpha} X (UV^T)^T - I)^i$
\end{proposition}

\begin{proof}
Let $G = U\Sigma V^T$ be the singular value decomposition of $G$. The Schatten-$(p+1)$ fractional map is defined by
$$\mathcal{F}_p(G) = U\Sigma^{1/p}V^T.$$

Let $P = UV^T$ denote the polar factor of $G$, and let $\alpha > 0$ be a normalization factor (usually $\alpha \geq \|G\|_2$ with $\|\cdot\|_2$ the spectral norm). Define
$$Z = \frac{1}{\alpha} U\Sigma U^T .$$
Note that the eigenvalues of $Z$ are between 0 and 1, due to the normalization factor $\alpha$. Using $G = U\Sigma V^T$ and $P^T = VU^T$, we have
$$Z = \frac{1}{\alpha} U\Sigma U^T = \frac{1}{\alpha} G P^T .$$

Now observe that, since $Z$ is a positive semi-definite matrix,
$$
    Z^{1/p} = \left(\frac{1}{\alpha} U\Sigma U^T \right)^{1/p} = \frac{1}{\alpha^{1/p}} U\Sigma^{1/p}U^T .
$$
Therefore,
$$
    \alpha^{1/p} Z^{1/p} P = \alpha^{1/p} \left( \frac{1}{\alpha^{1/p}} U\Sigma^{1/p}U^T \right) UV^T = U\Sigma^{1/p}V^T .
$$
Hence,
$$ \mathcal{F}_p(G) = \alpha^{1/p} Z^{1/p}P .$$

Next, with $I$ the identity matrix and $E$ the complementary matrix to $Z$, write
$$ Z = I + E,\qquad E = Z - I .$$
Then, using the binomial Taylor expansion,
$$ 
    Z^{1/p} = (I+E)^{1/p} = \sum_{i=0}^{\infty} \binom{1/p}{i} E^i .
$$
Since
$$ E = Z-I = \frac{1}{\alpha}GP^T - I, $$
we obtain
$$ Z^{1/p} = \sum_{i=0}^{\infty} \binom{1/p}{i} \left( \frac{1}{\alpha}GP^T - I \right)^i .$$

Finally we obtain the result in~\cref{prop:matrix_approx},
\[
    \mathcal{F}_p(G)
    =
    \alpha^{1/p}
    \left[
        \sum_{i=0}^{\infty}
        \binom{1/p}{i}
        \left(
            \frac{1}{\alpha}GP^T - I
        \right)^i
    \right]
    P .
\]

In practice, the infinite series is truncated at degree \(K\), yielding the approximation
\[
    \widehat{\mathcal{F}}_p(G)
    =
    \alpha^{1/p}
    \left[
        \sum_{i=0}^{K}
        \binom{1/p}{i}
        \left(
            \frac{1}{\alpha}GP^T - I
        \right)^i
    \right]
    P .
\]

Moreover, computing exactly the polar factor $P$ is prohibitive. The Newton-Schulz approximation is used to efficiently estimate $P$. 
\end{proof}

%% file: app/02_second_order_p.tex
\section{Second-Order Moments in the Schatten-p Framework}
\label{app:second_order_p}

The main text develops the optimal Schatten exponent $p^*$ criterion for $D = \mathbf{1}$ (SGDM, Muon, and Schatten-$p$ momentum descent). This appendix shows how an elementwise second-order preconditioner $D$ can be fitted into the framework, characterizes the resulting modification to the $p^*$ criterion, and discusses the interpolation between Euclidean and spectral families that arises naturally from this construction.

\paragraph{Remark} Note that the formula for the element-wise preconditioned proxy optimality criterion is not used in the paper experiments due to poor scalability, this form is provided purely for completeness. 

\subsection{Plugging Adam-Style Coordinate Rescaling into the Update Rule}
\label{app:adam-plugin}

Recall the unified update rule from~\citet{veprikov2025preconditionednormsunifiedframework}:
\begin{equation}
    \delta W = -\eta\, D^{\circ -1} \odot \mathrm{LMO}_{(p+1)}\bigl(D^{\circ -1} \odot M\bigr),
\label{eq:appendix-general-update}
\end{equation}
where $D \in \mathbb{R}^{m \times n}_{>0}$ encodes second-order moment information. For Adam-style preconditioning, $D^{\circ -1}$ is constructed from a bias-corrected exponential moving average of squared gradients $\hat{V}_t = (1 - \beta_2^t)^{-1} \sum_{s \leq t} \beta_2^{t-s}(1 - \beta_2)\, G_s \odot G_s$, with the multiplier scaling
\begin{equation}
    D^{\circ -1} = (\hat{V} + \varepsilon)^{\circ -\alpha},
\label{eq:appendix-adam-D}
\end{equation}
where $\alpha = 1/2$ recovers Adam-style RMS rescaling, $\alpha = 1/4$ recovers MuAdam~\citep{veprikov2025preconditionednormsunifiedframework}, and $\alpha = 1/(2(p+1))$ defines a \emph{$p$-dependent} family that smoothly transitions from MuAdam at $p = 1$ to Muon at $p \to \infty$ (see~\cref{app:interpolation}).

The recipe is mechanical: pre-multiply $G$ and $M$ elementwise by $D^{\circ -1}$ before invoking the LMO, and post-multiply the LMO output again by $D^{\circ -1}$ before applying the update. This is what \citet{veprikov2025preconditionednormsunifiedframework} call the \emph{symmetric D-norm preconditioned LMO} for the Schatten-$(p+1)$ base norm.

\subsection{Modification of the Optimal p Proxy}
\label{app:p-star-with-D}

We now derive how the $p^*$ criterion of~\cref{prop:p_star_moment_tight} changes when $D \neq \mathbf{1}$. Write $\tilde{G} := D^{\circ -1} \odot G$, $\tilde{M} := D^{\circ -1} \odot M$, and let $\tilde{M} = \tilde{U}_M \tilde{\Sigma}_M \tilde{V}_M^T$ be its SVD. The actual physical update is
\begin{equation}
    \delta W = -\eta\, D^{\circ -1} \odot \tilde{Y}_p, \qquad \tilde{Y}_p := \tilde{U}_M \tilde{\Sigma}_M^{1/p} \tilde{V}_M^T.
\label{eq:appendix-physical-update}
\end{equation}

\paragraph{Alignment term.} Using the identity $\langle G, D^{\circ -1} \odot Y\rangle_F = \langle \tilde{G}, Y\rangle_F$, the first-order term is
\begin{equation}
    \langle G, \delta W\rangle_F = -\eta \sum_i \tilde{\sigma}_{M,i}^{1/p}\, \tilde{C}_{i,i}, \qquad \tilde{C}_{i,i} := (\tilde{U}_M^T \tilde{G} \tilde{V}_M)_{i,i}.
\label{eq:appendix-alignment}
\end{equation}
This expression is exact for arbitrary $D$ and is independent of $p$ except through the $\tilde{\sigma}_{M,i}^{1/p}$ factor.

\paragraph{Curvature term.} The curvature is
\begin{equation}
    \frac{1}{2n}\|\delta W \cdot A\|_F^2 = \frac{\eta^2}{2n}\bigl\|(D^{\circ -1} \odot \tilde{Y}_p) A\bigr\|_F^2.
\label{eq:appendix-curvature-exact}
\end{equation}
Crucially, the Hadamard product does not commute with matrix multiplication: $(D^{\circ -1} \odot \tilde{Y}_p) A \neq D^{\circ -1} \odot (\tilde{Y}_p A)$. Expanding in the singular basis,
\begin{equation}
    \bigl\|(D^{\circ -1} \odot \tilde{Y}_p) A\bigr\|_F^2 = \sum_{r,s} \tilde{\sigma}_{M,r}^{1/p}\, \tilde{\sigma}_{M,s}^{1/p}\, K_{rs}(D),
\label{eq:appendix-curvature-cross}
\end{equation}
where the coupling tensor is
\begin{equation}
    K_{rs}(D) := \sum_{i, j, k} (D^{\circ -1})_{ij} (D^{\circ -1})_{ik}\, \tilde{U}_{M,ir} \tilde{U}_{M,is}\, \tilde{V}_{M,jr} \tilde{V}_{M,ks} (A A^T)_{jk}.
\label{eq:appendix-K-tensor}
\end{equation}
The diagonal entries $K_{rr}$ recover the $D = \mathbf{1}$ result $\tilde{B}_{i,i} := \|A^T \tilde{V}_M e_i\|^2$ when $D \propto \mathbf{1}$. The off-diagonal entries $K_{rs}$ for $r \neq s$ vanish identically in that case (by orthonormality of $\tilde{U}_M$) but are nonzero in general, they are the \emph{cross-singular-value couplings} between distinct components of $\tilde{Y}_p$ induced by the elementwise preconditioner.

\paragraph{Computational consequence.} For $D = \mathbf{1}$, the curvature is a sum $\sum_i \tilde{\sigma}_{M,i}^{2/p}\, \tilde{B}_{i,i}$, allowing the bounded scalar minimization over $p$ to be evaluated from precomputed quantities $\{\tilde{\sigma}_{M,i}, \tilde{C}_{i,i}, \tilde{B}_{i,i}\}_i$ at $O(k)$ cost per call. 

For $D \neq \mathbf{1}$, the curvature~\eqref{eq:appendix-curvature-exact} requires reforming $\tilde{Y}_p$ for each candidate $p$ specifically.

\begin{equation}
    D_\text{val}(p) = \bigl\|(D^{\circ -1} \odot \tilde{U}_M \tilde{\Sigma}_M^{1/p} \tilde{V}_M^T) A\bigr\|_F^2,
\label{eq:appendix-Dval-recompute}
\end{equation}
which costs one $m \times n$ matrix product to form $\tilde{Y}_p$, one Hadamard product, one $m \times s$ matrix product with $A$, and one Frobenius norm. Compared to the $D = \mathbf{1}$ closed form, this is an order-of-magnitude increase in per-call cost during the bounded scalar search, but remains tractable: the bounded minimizer typically issues $\sim 20$ objective evaluations.
% costing roughly the same as a single forward pass on the layer being optimized.

\begin{proposition}[Layerwise optimal Schatten exponent under preconditioning]
\label{prop:p_star_with_D}
Let $\tilde{M} = D^{\circ -1} \odot M$ have SVD $\tilde{U}_M \tilde{\Sigma}_M \tilde{V}_M^T$, and let $\tilde{C}_{i,i}$ be defined as in~\eqref{eq:appendix-alignment}. The value of $p$ that maximizes the exact one-step descent guarantee on the random feature regression loss under co-optimal step size and the update rule~\eqref{eq:appendix-physical-update} is
\begin{equation}
    p^*(D) = \operatorname{argmax}_{p \in [1, \infty]}\; \frac{\bigl(\sum_i \tilde{\sigma}_{M,i}^{1/p}\, \tilde{C}_{i,i}\bigr)^2}{\bigl\|(D^{\circ -1} \odot \tilde{U}_M \tilde{\Sigma}_M^{1/p} \tilde{V}_M^T) A\bigr\|_F^2}.
\label{eq:appendix-p-star-with-D}
\end{equation}
For $D \propto \mathbf{1}$, this reduces to Eq.~\eqref{eq:optimal_p_first_order} of the main text, with the curvature decomposing into a sum over singular values. For general $D$, the curvature is irreducibly a quadratic form in $\tilde{\sigma}_{M,r}^{1/p}\tilde{\sigma}_{M,s}^{1/p}$, and the bounded scalar search requires reforming $\tilde{Y}_p$ at each candidate $p$.
\end{proposition}

In practice however, we consider that the optimality of the $p^*$ proxy value is sufficient by only taking first order moments into account, and we leave the choice of $p$-dependent second order moment as a consolidating heuristic for the optimization. 
Importantly, this heuristic simplification is \textbf{valid}, as it simply assumes that the first-order moment random feature regression setting is \textit{sufficient} to model the behavior of the adaptive LMO in \textit{preconditioned} settings.

\subsection{Interpolation Between Euclidean and Spectral Families}
\label{app:interpolation}

A pleasant consequence of the unified update rule~\eqref{eq:appendix-general-update} is that varying $p$ and the choice of $D$ jointly traces out a 2D family of optimizers, recovering established methods at its corners. \Cref{tab:interpolation} summarizes this structure.

\begin{table}[h]
\centering
\renewcommand{\arraystretch}{1.4}
\setlength{\tabcolsep}{8pt}
\begin{tabular}{lccc}
\textbf{Moment configuration} & \textbf{$p = 1$ update} & \textbf{$p \to \infty$ update} \\
\midrule
$D = \mathbf{1}$ (no moment) & SGDM & Muon \\
$D^{\circ -1} = \hat{V}^{\circ -1/4}$ (fixed) & Adam & MuAdam \\
$D^{\circ -1} = \hat{V}^{\circ -1/(2(p+1))}$ ($p$-dep.) & Adam & Muon \\
\bottomrule
\vspace{5pt}
\end{tabular}
\caption{Optimizer family recovered by the unified update rule~\eqref{eq:appendix-general-update} for various choices of moment $D$ and Schatten exponent $p$. Fixed-$\alpha$ moments interpolate the \emph{magnitude} of the update toward Adam-style at $p = 1$ but preserve the Muon-like spectral structure at $p \to \infty$ \emph{with persistent variance rescaling}. $p$-dependent moments additionally have $D \to \mathbf{1}$ as $p \to \infty$, recovering pure Muon at the spectral endpoint and pure MuAdam at the Euclidean endpoint. 
% Both fixed and $p$-dependent recipes are compatible with structured second-moment estimators including AdaFactor's rank-one $\hat{V} \approx r c^T$ and SANIA-style Polyak rescalings.
}
\label{tab:interpolation}
\end{table}

The $p$-dependent family is particularly attractive: a single optimizer traces out the entire interpolation between MuAdam and Muon as $p$ varies. The adaptive $p^*$ selection of~\Cref{sec:rfr} then dynamically chooses the right point in this interpolation per layer per training phase, without manual specification.
Additionally, this framework can be combined more efficiently with different structures of $D$ scaling, where, if $D$ is a diagonal matrix multiplication, the efficient computation of $p^*$ can be reinstated from trace operator properties.

In practice, to reconcile the different update scales of Muon and Adam for example, in the SMuon (Adam) optimizer: we interpolate the learning rate logarithmically between that of Muon using the standard default implementation multiplicative factor (when $p^*= p_{\max}$), and that of the auxiliary AdamW optimizer that is used on non-matrix parameters (when $p^*=1$), ensuring that the empirical RMS norm of updates remains rather stable across $p^*$ values. This eliminates the need for an additional hyperparameter as most practical training scenarios using Muon use auxiliary Adam updates on parameters such as biases. Moreover, while analytical scaling factors could be found for the scale of the updates, they would require costly Schatten-$p^*$ norm estimations at each step.

\subsection{Recovering the Right Exponent}

We show that $\alpha = 1/(2(p+1))$ is the \emph{unique} exponent such that the symmetric preconditioned update
\begin{equation}
    \delta W = -\eta\, D^{\circ -\alpha} \odot \mathrm{LMO}_{(p+1)}\!\bigl(D^{\circ -\alpha} \odot M\bigr),
    \label{eq:symm_update}
\end{equation}
simultaneously (i) recovers the Adam update at $p=1$ and (ii) is invariant to global rescaling of the loss.

\begin{proof}
\textbf{Uniqueness via loss-rescaling invariance.}
Under a global rescaling $L \to \lambda L$ the gradient, momentum, and second-moment accumulator transform as
\[
    G \;\to\; \lambda G, \qquad M \;\to\; \lambda M, \qquad D \;\to\; \lambda^2 D,
\]
since $D$ is an EMA of $G \odot G$. Substituting into \eqref{eq:symm_update},
\[
    \delta W \;\to\;
    -\eta\,\lambda^{-2\alpha} D^{\circ -\alpha}
    \odot \mathrm{LMO}_{(p+1)}\!\bigl(\lambda^{1-2\alpha}\, D^{\circ -\alpha} \odot M\bigr).
\]
The Schatten-$p$ LMO is positively homogeneous of degree $1/p$, i.e.\ $\mathrm{LMO}_{(p+1)}(\mu X) = \mu^{1/p}\mathrm{LMO}_{(p+1)}(X)$ for $\mu>0$, so the expression reduces to
\[
    \delta W \;\to\;
    -\eta\,\lambda^{\,-2\alpha\,+\,(1-2\alpha)/p}\;
    D^{\circ -\alpha} \odot \mathrm{LMO}_{(p+1)}\!\bigl(D^{\circ -\alpha} \odot M\bigr).
\]
For the update \emph{direction} to be independent of $\lambda$ (absorbing any scalar into the step size $\eta$), the exponent of $\lambda$ must vanish:
\[
    -2\alpha + \frac{1-2\alpha}{p} = 0
    \;\;\Longrightarrow\;\;
    \alpha = \frac{1}{2(p+1)}.
\]
This equation has a unique solution for every $p\ge 1$, establishing necessity and sufficiency.

\textbf{Recovery of Adam at $p=1$.}
Setting $p=1$ gives $\alpha = \tfrac{1}{4}$. At $p=1$, the Schatten-2 LMO satisfies $\mathrm{LMO}_2(X) \propto X$ (the update is proportional to the matrix itself), so
\[
    \delta W \;\propto\;
    D^{\circ -1/4} \odot \bigl(D^{\circ -1/4} \odot M\bigr)
    \;=\; D^{\circ -1/2} \odot M,
\]
which is exactly the Adam preconditioned update with RMS denominator $D^{1/2}$.

\textbf{Recovery of Muon as $p\to\infty$.}
As $p\to\infty$, $\alpha = 1/(2(p+1))\to 0$, so $D^{\circ -\alpha}\to \mathbf{1}$ and the preconditioning vanishes. The update reduces to $\mathrm{LMO}_\infty(M) = UV^T$, recovering the pure Muon orthogonalization step.
\end{proof}

%% file: app/03_scaling_behaviour.tex
\section{Scaling Behavior across Adaptive Geometries}
\label{app:scaling}

\subsection{Behavior Across Model Scale and Depth}

We examine whether the adaptive geometry remains meaningful as the
underlying model changes in scale. We first increase the language-model
configuration from $275.7$M to $716.6$M parameters by changing the width of the blocks in the nanogpt architecture while keeping depth fixed (see \Cref{tab:scaling_lm} for validation loss with different model sizes).
For plain SMuon, the time-averaged selected exponent increases from
$28.58$ to $37.43$, while the fraction of layerwise selections satisfying
$p^* \leq 10$ decreases from $39.5\%$ to $21.8\%$. Thus, increasing width
shifts the optimizer toward more Muon-like proxy-optimal geometries 
which corroborates common knowledge about Muon's scalability~\citep{liu2025muon}.

\begin{table}[t]
\centering
\caption{Final / time-averaged validation loss as model size increases.
Lower is better, averaged across three random seeds.}
\label{tab:scaling_lm}
\begin{tabular}{lccc}
\toprule
Parameters & Muon & SMuon & SMuon (Adam) \\
\midrule
$275.7$M & $3.3781 / 3.9373$ & $3.3828 / 3.9276$ & $\mathbf{3.37156/ 3.9130}$ \\
$716.6$M & $3.2237 / 3.9126$ & $3.2179 / 3.8815$ & $\mathbf{3.2096 / 3.8735}$ \\
\bottomrule
\end{tabular}
\end{table}

Changing depth produces a different response. At fixed width, increasing the
model from $6$ to $24$ blocks decreases plain SMuon's mean $p^*$ from
$38.47$ to $24.42$. 
These observations suggest that model parameter count alone is insufficient
for geometry adaptation and model-dependent knowledge is needed. 

The practical computational overhead also remains small at the larger scale.
With the selector refreshed every $100$ steps, SMuon (Adam) increases measured
training-step time from $1336.0$ to $1375.6$ ms at $275.7$M parameters
($2.96\%$ overhead), and from $3381.1$ to $3508.0$ ms at $716.6$M
parameters ($3.75\%$ overhead).

\subsection{Batch-Size Ablation}
\label{app:batch_size}

We additionally test whether the adaptive behavior persists across batch sizes
on MLP-Mixer/ImageNette. In \Cref{tab:batch_size_ablation}, we compare batch sizes $128$ and $512$, approximately
matching the optimizer-step budget ($1460$ and $1440$ updates, respectively).
For each optimizer and batch-size pair, the learning rate is independently selected
from six candidates, after which the chosen configuration is evaluated over
four matched seeds with identical initialization and data loading.

\begin{table}[h!]
\centering
\caption{Final ImageNette validation accuracy (\%) across batch sizes.
Values are mean $\pm$ standard deviation over four matched seeds.}
\label{tab:batch_size_ablation}
\begin{tabular}{lccc}
\toprule
Batch size & Muon & SMuon & SMuon (Adam) \\
\midrule
$128$ & $77.12\pm0.24$ & $78.14\pm0.53$ & $\mathbf{78.67\pm0.48}$ \\
$512$ & $77.53\pm0.23$ & $\mathbf{78.60\pm0.48}$ & $78.10\pm0.81$ \\
\bottomrule
\end{tabular}
\end{table}

Plain SMuon improves over Muon by $1.02$ percentage points at batch size
$128$ and $1.07$ points at batch size $512$. SMuon (Adam) improves over
Muon by $1.55$ and $0.57$ points, respectively. The optimal $p^*$ values also
change with batch size: for plain SMuon, the fraction of decisions for all layers
and timesteps satisfying $p^*<8$ increases from $5.0\%$ at batch size $128$ to $21.4\%$ at
batch size $512$. These lower-$p^*$ selections are concentrated in the
token-mixing output projections and classification head, while the stem and
channel-mixing projections remain predominantly Muon-like.

\subsection{Distributed Training Considerations}

In distributed settings, our $p^*$ computation can be efficiently adapted. While we do not discuss practical implementation details, we provide a strategy for the computation of $p^*$ when activation matrices $A$ are sharded across different devices. Instead of communicating all activations to a single rank, our distributed implementation computes the local projected Gram contributions on each rank, which can then be combined via a single all-reduce of size $\mathcal{O}(\min(m, n))$ to correctly capture the global activation geometry.

More precisely, given activations $A$ sharded along the batch dimension as $A = [A_1, \dots, A_R]$ across $R$ ranks (with $A_r \in \mathbb{R}^{n \times k_r}$ the local shard), the curvature term $B_{i,i} = \|A^T V_M e_i\|^2$ admits the decomposition $\|V_M^\top A\|_F^2 = \sum_r \|V_M^\top A_r\|_F^2$, so each rank can compute its local contribution $V_M^\top A_r A_r^\top V_M$ and a single all-reduce of size $\mathcal{O}(\min(m, n)^2)$ or $\mathcal{O}(\min(m, n))$ if only the diagonal is required  recovers the global $B_{i,i}$. The momentum SVD $U_M \Sigma_M V_M^\top$ is computed redundantly on each rank from the (already-replicated) momentum buffer, incurring no additional communication. The alignment coefficients $C_{i,i} = (U_M^\top G_t V_M)_{i,i}$ are similarly local once gradients are reduced, which any data-parallel training loop already does. The total $p^*$-induced communication overhead is therefore one all-reduce of a vector or small matrix of size $\mathcal{O}(\min(m, n))$ per selector call, which at our default update interval of 100 steps is negligible relative to the per-step gradient all-reduce.

%% file: app/04_fixedp_vs_adaptive.tex
\section{Adaptive versus Fixed Schatten Geometry}
\label{app:fixed_p}

Here, we want to test whether the benefit of SMuon comes from adapting the
Schatten geometry itself, or whether similar performance can be obtained by
choosing a single fixed exponent and retaining the same surrounding optimizer.
We therefore compare adaptive $p^*$ selection against fixed-$p$ controls in
both language modeling and vision.

\subsection{NanoGPT Training Task}
\label{app:fixed_p_nanogpt}

In \Cref{tab:fixed_p_nanogpt}, we compare SMuon (Adam) against an otherwise identical variant in
which every matrix parameter uses a fixed exponent $p=50$ throughout
training, as using $p = 10$ or $p = 40$ degraded the performance. 
Since the NanoGPT experiment is a speedrun setting, we evaluate convergence by
the number of optimizer steps required to reach common validation-loss
thresholds. For six matched seeds, threshold-crossing times are
estimated by linear interpolation between logged validation points.

\begin{table}[h!]
\centering
\caption{Adaptive versus fixed-$p$ SMuon (Adam) on NanoGPT. Threshold
crossings are averaged over six matched seeds. ``Saved'' reports the reduction
in optimizer steps obtained by adaptive $p^*$.}
\label{tab:fixed_p_nanogpt}
\begin{tabular}{lrrrr}
\toprule
Target loss & Fixed $p=50$ & Adaptive $p^*$ & Steps saved & Relative \\
\midrule
$4.20$ & $259.5$  & $252.2$  & $7.3$  & $2.83\%$ \\
$4.00$ & $367.3$  & $358.2$  & $9.1$  & $2.49\%$ \\
$3.80$ & $565.1$  & $548.0$  & $17.1$ & $3.02\%$ \\
$3.60$ & $986.7$  & $962.9$  & $23.9$ & $2.42\%$ \\
$3.40$ & $1545.3$ & $1507.8$ & $37.5$ & $2.43\%$ \\
\bottomrule
\end{tabular}
\end{table}

Adaptive SMuon (Adam) reaches every tested threshold earlier in all six
matched seeds. It also reduces the time-averaged validation loss from
$3.9173$ to $3.9048$, and improves the final validation loss from $3.3754$
to $3.3663$. Because the fixed and adaptive variants share the same
$p$-dependent preconditioning mechanism, this comparison isolates a measurable
convergence benefit from adapting the geometry itself.

\subsection{Vision Training Task}
\label{app:fixed_p_vision}

We next test whether a single manually chosen exponent can recover the
performance of adaptive selection in vision. We train an MLP-Mixer on
ImageNette for $100$ epochs and compare adaptive layerwise $p^*$ against
fixed values
\[
p\in\{1.02,3,8\}.
\]
Each fixed-$p$ baseline is tuned independently.

The tuning stage uses a specific seed for $30$ epochs. We first
perform a coarse search over
\[
\eta_{\mathrm{Muon}}\in\{0.005,0.01,0.02,0.04\},
\qquad
\eta_{\mathrm{Adam}}\in\{3\times10^{-4},10^{-3},3\times10^{-3}\},
\]
followed by local refinement over
\[
\beta_1\in\{0.9,0.95\},\qquad
\beta_2\in\{0.99,0.999\},
\]

For the SMuon parameters, we use $\beta_{\mathrm{sv}} = 0.9$,
and the interval for $p^*$ computations is fixed to $30$.
The remaining settings are shared across the search:
$\beta_1=0.95$, $\epsilon=10^{-8}$,
$\eta_{\min}=10^{-7}$, $p_{\min}=1.02$, $p_{\max}=50$.
% The selected configurations are shown in
% Table~\ref{tab:fixed_p_vision_hparams}.

% \begin{table}[t]
% \centering
% \caption{Selected configurations for the adaptive and fixed-$p$ MLP-Mixer
% experiments.}
% \label{tab:fixed_p_vision_hparams}
% \resizebox{\linewidth}{!}{
% \begin{tabular}{lccccc}
% \toprule
% Method & $\eta_{\mathrm{Muon}}$ & $\eta_{\mathrm{Adam}}$ &
% Muon momentum & $\beta_2$ & SV momentum / interval \\
% \midrule
% Adaptive $p^\star$ & $0.02$ & $0.003$  & $0.90$ & $0.99$ & $0.90 / 30$ \\
% Fixed $p=1.02$     & $0.02$ & $0.001$  & $0.90$ & $0.99$ & $0.90 / 10$ \\
% Fixed $p=3$        & $0.04$ & $0.003$  & $0.90$ & $0.99$ & $0.90 / 10$ \\
% Fixed $p=8$        & $0.04$ & $0.0003$ & $0.95$ & $0.99$ & $0.90 / 10$ \\
% \bottomrule
% \end{tabular}}
% \end{table}

\begin{table}[t]
\centering
\caption{Adaptive versus independently tuned fixed Schatten exponents on
MLP-Mixer/ImageNette after $100$ epochs. Values are mean $\pm$ standard
deviation over nine common seeds.}
\label{tab:fixed_p_vision_results}
\begin{tabular}{lcc}
\toprule
Method & Final train loss & Final validation accuracy (\%) \\
\midrule
Adaptive $p^*$ & $\mathbf{0.00540\pm0.00179}$ & $82.66\pm0.30$ \\
Fixed $p=1.02$     & $0.02575\pm0.00422$ & $81.64\pm0.37$ \\
Fixed $p=3$        & $0.00613\pm0.00195$ & $\mathbf{82.74\pm0.40}$ \\
Fixed $p=8$        & $0.00729\pm0.00257$ & $82.07\pm0.41$ \\
\bottomrule
\end{tabular}
\end{table}

Results are given in \Cref{tab:fixed_p_vision_results}. Adaptive $p^*$ obtains the lowest final training loss while reaching
validation accuracy comparable to the best independently tuned fixed geometry.
In particular, fixed $p=3$ attains a slightly higher mean validation accuracy,
whereas adaptive selection gives lower training loss and does not require
choosing $p$ through an external grid search. 
This shows that the RFR selector can recover performance
comparable to the tuned fixed exponent while removing the Schatten
exponent itself from the hyperparameter search.

Taken together with the NanoGPT control, these experiments make the case for 
geometry adaptation for out-of-the-box performance without LMO tuning. 

%% file: app/05_p_stability.tex
\section{On the Stability of Optimal p Values}
\label{app:p_stability}

In this section, we analyze the empirical stability of the optimal Schatten exponent $p^{*}$ throughout the training process. Observations demonstrate that the trajectory of $p^{*}$ values remains relatively stable over time. Although the precise $p^{*}$ values are not strictly identical across independent runs,~\Cref{fig:pstar-nanogpt} illustrates that specific layer types consistently exhibit similar behavioral patterns regardless of the random seed. 

Particularly, distinct architectural components naturally gravitate toward different optimal update geometries:
\begin{itemize}
    \item \textbf{Attention Weights:} These matrices typically require higher-rank updates, closely mirroring Adam-like update geometries at the beginning of the training process. 
    \item \textbf{MLP Layers:} In contrast, MLP layers generally necessitate slightly lower-rank updates throughout training, while still maintaining an update geometry that is very close to strict orthogonality.
\end{itemize}

This consistent structural divergence between layer types further highlights the practical benefit of utilizing a dynamically adapted, layer-wise $p^{*}$ selector over fixed global hyperparameter choices.

\begin{figure}
    \centering
    \includegraphics[width=1.0\linewidth]{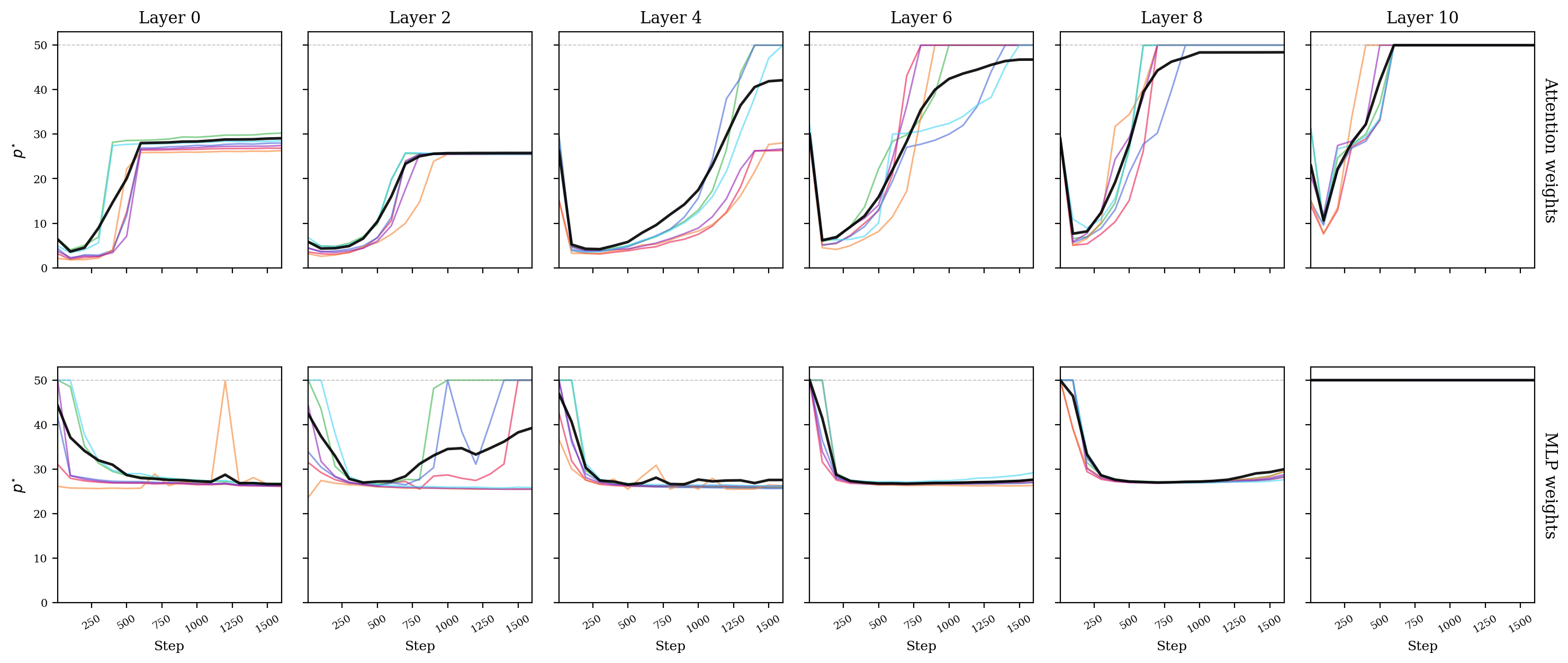}
    \caption{Here we plot the average $p^*$ values of attention and MLP layer weight matrices during training with the SMuon (Adam) optimizer across 6 random seeds, $p^*$ values are computed every 100 steps with an EMA of 0.95.}
    \label{fig:pstar-nanogpt}
\end{figure}

\subsection{Ablation Study: Update Frequency and EMA Regularization}
\label{sec:ablation}

We conducted an ablation study to investigate the interplay between update frequency (measured by the step gap) and the Exponential Moving Average decay rate on training stability and final convergence. In this case, we train an MLP-Mixer - Small model on the ImageNette dataset using a batch size of 64 and a learning rate of $0.02$ with the SMuon (Adam) optimizer. The results, averaged across three random seeds, are summarized in~\Cref{tab:ema_ablation}.

The empirical data demonstrates that the optimal EMA decay rate is highly dependent on the update frequency. Models operating with highly frequent updates (Gap = 50) exhibited significant sensitivity to the EMA parameter. Specifically, in these high-noise regimes, a strong EMA ($\beta_p = 0.9$) was critical for stabilizing the training process, achieving a lower final loss ($0.0069$) compared to lower EMA values ($0.0088$ for $\beta = 0.0$). Conversely, when updates were less frequent (Gap $\ge$ 300), the larger accumulated data batches were inherently more stable, and moderate-to-low EMA values proved optimal (e.g., $0.0060$ for Gap 600, $\beta = 0.5$).

While moderate EMAs suffice for low-frequency updates in smaller-scale tests, larger models necessitate a stable tracking mechanism for a significantly higher number of parameter matrices. Thus enforcing a stricter constraint on the EMA strength needed for stable training. 

\begin{table}[htpb]
    \centering
    \caption{Final training loss across different update gaps and EMA values, averaged over three random seeds. In high-frequency update regimes (Gap = 50), higher EMA values are strictly required to stabilize the loss.}
    \label{tab:ema_ablation}
    \begin{tabular}{lccc}
        \toprule
        \textbf{Update Gap} & \textbf{$\beta_p$ = 0.0} & \textbf{$\beta_p$ = 0.5} & \textbf{$\beta_p$ = 0.9} \\
        \midrule
        50  & 0.0088 & 0.0086 & \textbf{0.0069} \\
        150 & 0.0073 & \textbf{0.0068} & 0.0074 \\
        300 & \textbf{0.0062} & 0.0074 & 0.0069 \\
        600 & 0.0063 & \textbf{0.0060} & 0.0063 \\
        \bottomrule
    \end{tabular}
\end{table}

\section{Further Insights on the Proxy-Optimality Model}
\label{app:freon}

\citet{shumaylov2026muon} study general spectral geometries and, through their \emph{Optimal C} analysis, investigate locally favorable spectral exponents using a continuous-exponential adapted random-feature proxy (similar to that of \cite{massena_pmuon}'s blog post). Their controlled experiments show that greedily following single-step optima can be limiting over longer horizons, while their GPT-2 analysis reveals substantial batch-to-batch variability in the preferred geometry.

SMuon addresses these practical issues with a momentum-aware proxy and an exponential moving average over the spectral statistics used to estimate $p^*$. SMuon (Adam) additionally introduces $p$-dependent second-order scaling, providing coordinate-wise variance control when the selected geometry moves away from the Muon regime.
To isolate these effects, we compare against a simplified \emph{naive Optimal-C} proxy with neither momentum-awareness nor EMA on spectral statistics, that uses the same interval for layer-wise $p^*$ computations. This is a controlled ablation of the corresponding proxy-optimality mechanism rather than a reproduction of Freon or Optimal C.

We evaluate convergence speed directly through validation-loss thresholds. For each run, we record
the first timestep at which a target validation loss is reached, using linear interpolation between
consecutive validation measurements. Results, in \Cref{tab:optimal-c-convergence}, are averaged over six matched seeds.

\begin{table*}[t]
    \centering
    \caption{
        Mean timestep required to first reach each validation-loss threshold over six matched seeds.
        Lower is better. Relative reductions are measured against the naive Optimal-C proxy.
    }
    \label{tab:optimal-c-convergence}
    \resizebox{\textwidth}{!}{
    \begin{tabular}{lrrrrrrrrr}
        \toprule
        \textbf{Target loss}
        & \textbf{5.00}
        & \textbf{4.50}
        & \textbf{4.00}
        & \textbf{3.80}
        & \textbf{3.60}
        & \textbf{3.50}
        & \textbf{3.45}
        & \textbf{3.42}
        & \textbf{3.40} \\
        \midrule
        Naive Optimal-C
        & 131.9 & 206.0 & 390.6 & 594.6 & 1011.7
        & 1241.0 & 1382.2 & 1482.3 & 1558.0 \\
        SMuon
        & 117.5 & 183.5 & 365.8 & 562.1 & 981.1
        & 1218.9 & 1361.0 & 1465.5 & 1541.8 \\
        SMuon (Adam)
        & $\mathbf{115.9}$ & $\mathbf{181.3}$ & $\mathbf{358.2}$ & $\mathbf{548.0}$
        & $\mathbf{962.9}$ & $\mathbf{1197.7}$ & $\mathbf{1335.6}$
        & $\mathbf{1436.8}$ & $\mathbf{1507.8}$ \\
        \midrule
        SMuon reduction
        & 10.9\% & 11.0\% & 6.4\% & 5.5\% & 3.0\%
        & 1.8\% & 1.5\% & 1.1\% & 1.0\% \\
        SMuon (Adam) reduction
        & $\mathbf{12.1\%}$ & $\mathbf{12.0\%}$ & $\mathbf{8.3\%}$ & $\mathbf{7.8\%}$
        & $\mathbf{4.8\%}$ & $\mathbf{3.5\%}$ & $\mathbf{3.4\%}$
        & $\mathbf{3.1\%}$ & $\mathbf{3.2\%}$ \\
        \bottomrule
    \end{tabular}}
\end{table*}

SMuon reaches every reported validation-loss threshold earlier than the naive proxy, with the
largest gains during the transient phase of training: approximately $11\%$ fewer steps to reach
losses $5.0$ and $4.5$, and $5$--$6\%$ fewer steps around losses $4.0$--$3.8$. As training enters
the flatter late-stage regime, this advantage decreases, consistent with the similar final losses of
the two methods.

SMuon (Adam) improves further at every threshold, reaching the same targets $12.1\%$, $12.0\%$,
$8.3\%$, and $7.8\%$ earlier at losses $5.0$, $4.5$, $4.0$, and $3.8$, respectively. Unlike SMuon,
it also retains a roughly $3$--$5\%$ advantage deeper into training. These results suggest that
momentum-aware, temporally smoothed geometric adaptation primarily improves convergence speed,
while $p$-dependent second-order scaling provides an additional complementary benefit that persists
closer to the late-training plateau.

%% file: app/06_hyperparameters.tex
% \section{AdamW Results}
% \label{app:adamw}

\section{Experiment Hyperparameters}
\label{app:hyperparams}

\paragraph{NanoGPT experiments}
Based on~\citet{du2026newtonmuonoptimizer}, our NanoGPT training setup processes 393,216 tokens per global step, achieved using a global batch size of 8 sequences with a sequence length of 49,152 ($48 \times 1024$). The models are trained for a total of 1,695 iterations on the FineWeb dataset. We use a 12-layer GPT architecture with 6 attention heads and an embedding dimension of 768. All models utilize a learning rate schedule with 0 warmup iterations and a linear cooldown over the final 762 iterations (45\% of total steps), with weight decay set to $0.0$.

For the SMuon and SMuon (Adam) optimizers, the hyperparameters remain unchanged from those that were optimal for Muon (which were already tuned). The Muon base learning rate is set to $0.05$ with momentum parameters of $0.95$. The auxiliary Adam optimizer applied to the embeddings, LM head, and scalars uses a learning rate of $0.008$ and $\beta = (0.8, 0.95)$. For the Schatten exponent updates in SMuon, we use an update interval of 100 steps, bounds of $p_{\min} = 1.02$ and $p_{\max} = 50.0$, conduct a $p^*$ evaluation every 100 steps with $\beta_p = 0.95$ to keep the training stable. SMuon (Adam) uses a $\beta_2$ value of $\beta_p = 0.95$.

For the standard AdamW baseline, we use the exact same learning rate ($0.008$) as that of the auxiliary Adam in the Muon setup.

For the MuAdam optimizer, we tuned the learning rate multiplier on the originally tuned learning rate across the grid $\{0.001, 0.01, 0.1, 0.5, 1.0, 1.5\}$ and selected the optimal value. Additionally, we experimented with the momentum parameters $(\beta_1, \beta_2)$ for MuAdam, testing both $(0.9, 0.999)$ and $(0.95, 0.95)$. We found the latter to perform more robustly and ended up sticking with $(\beta_1, \beta_2) = (0.95, 0.95)$ for our final configurations.

\paragraph{Vision experiments} We evaluate our optimization methods on image classification using the ImageNette dataset. We train two distinct architectures sourced from the \texttt{timm} library: a Vision Transformer (ViT-Small, \texttt{vit\_small\_patch16\_224}) and an MLP-Mixer (\texttt{mixer\_s16\_224}). All models are trained for $20$ epochs using a total batch size of $256$ in $\mathrm{bfloat16}$ mixed precision. Our data processing pipeline applies standard augmentations including RandAugment, random horizontal flipping, resizing to $256$, and center cropping to $224 \times 224$. The learning rate is decayed using a cosine annealing schedule without restarts. To ensure robust evaluation, all hyperparameter configurations are repeated across three random seeds (512, 1024, and 2048).

For the AdamW baseline, we sweep the learning rate over the grid $\{1.10^{-5}, 1.10^{-4}, 3.10^{-4}, 1.10^{-3}\}$ with a fixed weight decay of $0.01$. For the Muon, SMuon and SMuon (Adam) optimizers, we sweep the primary learning rate over $\{0.001, 0.003, 0.005, 0.01, 0.02\}$ for the ViT and $\{0.003, 0.005, 0.01, 0.02\}$ for the MLP-Mixer. The core Muon and SMuon parameters use a momentum of $0.95$ and a weight decay of $0.05$, with $\beta_2 = 0.99$ for SMuon (Adam). The auxiliary Adam optimizer applied to biases, normalizations, and other non-matrix parameters uses a fixed learning rate of $0.001$ and a weight decay of $0.01$. 
The singular value momentum is fixed to $0.95$.

\paragraph{LoRA Fine-Tuning} We evaluate our optimization methods on a low-rank adaptation (LoRA) fine-tuning task using the Qwen-2.5-0.5B-Instruct model on the GSM8K dataset. The base model is loaded in 4-bit quantization using $\mathrm{bfloat16}$ compute precision. We apply LoRA to both the attention and MLP projections (specifically \texttt{q\_proj}, \texttt{k\_proj}, \texttt{v\_proj}, \texttt{up\_proj}, \texttt{gate\_proj}, \texttt{query\_key\_value}, \texttt{dense\_h\_to\_4h}, and \texttt{dense\_4h\_to\_h}) with a rank of $r=32$ and $\alpha=64$. 

All models are trained for 3 epochs using a batch size of 32. We utilize a cosine learning rate schedule with a linear warmup over the first 10\% of total training steps. To ensure statistical robustness, all hyperparameter configurations are evaluated across 3 random seeds. 

For the AdamW baseline, as well as the Muon and SMuon optimizers, we sweep the primary learning rate across the grid $\{0.0003, 0.001, 0.003, 0.005, 0.01, 0.02\}$. For the Muon and SMuon configurations, the core 2D matrix parameters use a momentum of $0.95$ (and $\beta_2 = 0.95$ for the preconditioned version), while all 1D parameters are handled by an auxiliary Adam optimizer. 

For the SMuon variants, the Schatten exponent is initialized at the minimum bound, $p=1.02$ to avoid computational instabilities and updated using the method of~\Cref{prop:p_star_moment_tight}. We apply a Schatten update interval every epoch and a singular value momentum of $0.95$.

\subsection{Learning Rate Invariance}

An important part of evaluating optimizers in realistic deployment scenarios is to measure their sensitivity to hyperparameters. In practice, the Muon optimizer is also appreciated for its relatively stable behavior across learning rate values. In this setting, we train a ViT - small model and an MLP-Mixer small model with patch size 16 on the ImageNette dataset. In this experiment, the RandAugment augmentation method is unused to reduce variance in the training loss and measure the capability of the optimizer to fit the training set across different learning rates. Here, we use a batch size of $128$ with $p^*$ updates every epoch and $\beta_\mathrm{sv} = 0.95$. Results are displayed in~\cref{fig:lr-sensitivity}.

As seen in~\Cref{fig:lr-sensitivity}, the learning rate of SMuon and SMuon (Adam) is not radically different to that of Muon, with SMuon (Adam) only demonstrating slightly higher sensitivity to the learning rate on the MLP-Mixer - small model. 

\begin{figure}
    \centering
    \includegraphics[width=0.5\linewidth]{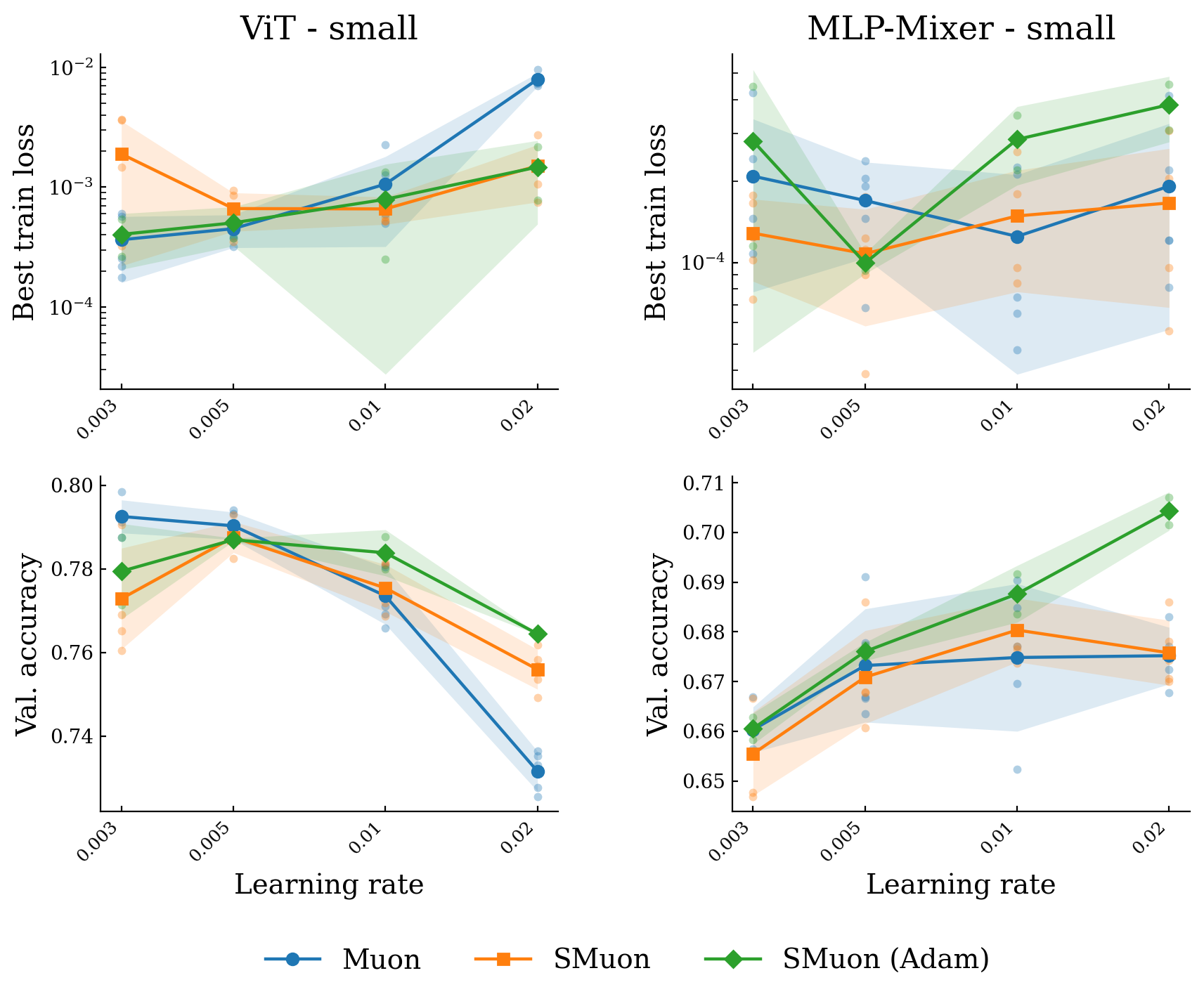}
    \caption{Learning rate invariance curves for two models trained on the ImageNette dataset}
    \label{fig:lr-sensitivity}
\end{figure}

%% file: app/07_retroactiveness.tex
\section{Investigating the Retroactive Effect of Using Adaptive Updates}
\label{app:retroaction}

In this rather exploratory section, we investigate the use of adaptive geometry updates. Our main endeavor is to investigate whether the geometry of updates at step $t$ can have an influence on the optimal geometry of updates at step $t+1$. To this end, we train an MLP-Mixer - Small model on the ImageNette dataset, using either Muon of SMuon. In this case, Muon is equipped with a $p^*$ approximator that runs 12 times at regular intervals on each experiment, that is ran on three different initialization strategies using three random seeds for each. Our results are presented in~\Cref{fig:pstar-retroaction}.

\Cref{fig:pstar-retroaction} reveals three convergence regimes.
$\mathrm{kaiming\_normal}$ and $\mathrm{xavier\_uniform}$ both favour Muon over SMuon ($d \approx -0.35$, $p \approx 0.07$ in each case), with divergence localised to $\mathrm{ch\_fc2}$ and $\mathrm{tok\_fc1}$ while $\mathrm{ch\_fc1}$ saturates at $p^{*} = 50$ for both optimisers.
Under $\mathrm{orthogonal}$ initialisation, both optimizers decline monotonically to $\bar{p}^{*} \approx 35$ with no detectable difference ($d \approx 0.03$, $p \approx 0.87$).
Initialisation explains more variance in $p^{*}$ than optimizer choice, and no pairwise difference survives Bonferroni correction ($\alpha_{\text{adj}} = 0.025$).

Therefore, it appears this experiment does not give any conclusive evidence of a retroactive effect between the optimizer update rule and the optimality of the geometry for the update at later training steps. However, this study, beyond single step dynamics, could prove particularly useful in the future to encourage the use of optimizers with strong positive retroactive effects.

Finally, this result can be interpreted positively, as strong retroactive effects between update rules and $p^*$ trajectories would severely limit the applicability of single step surrogate models. 

\begin{figure}
    \centering
    \includegraphics[width=0.85\linewidth]{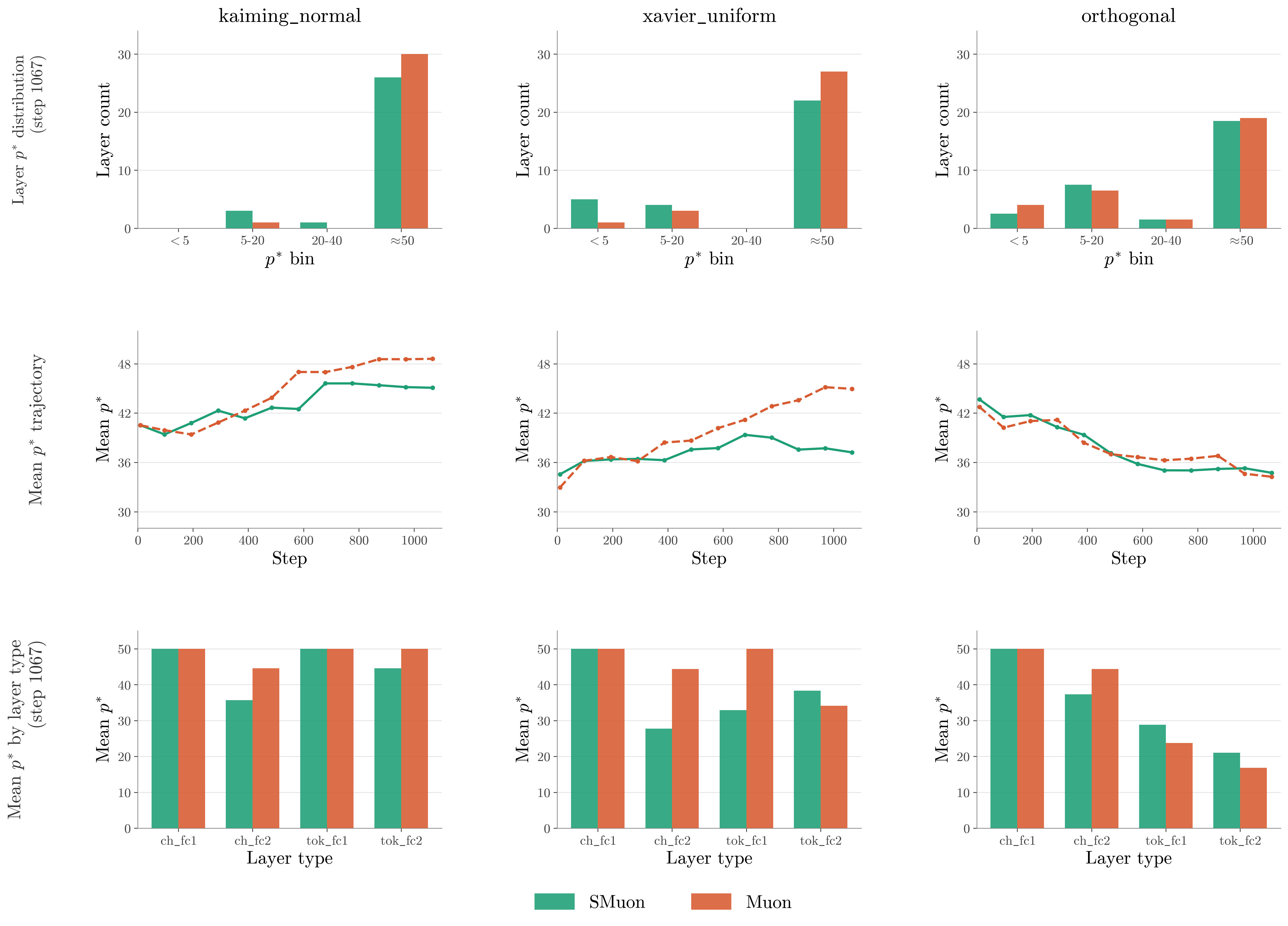}
    \caption{Comparison of $p^*$ dynamics between Muon and SMuon optimizers across three weight initialization strategies ($\mathrm{kaiming\_normal}$, $\mathrm{xavier\_uniform}$, $\mathrm{orthogonal}$). Each column corresponds to one initialization. Top row: distribution of final p* values across the 31 network layers, binned into four ranges. Middle row: mean $p^*$ trajectory over 12 checkpoints (steps 10–1067). Bottom row: mean final $p^*$ broken down by layer type (channel-mixing and token-mixing branches, first and second fully-connected layers).}
    \label{fig:pstar-retroaction}
\end{figure}

%% file: app/08_maximal_p.tex
\section{Alternative Computation Methods for the Schatten-p Update}
\label{app:remez}

\paragraph{Layerwise Remez-fitted Newton-Schulz.}
Here, we present a computational routine that extends the Björck/Newton-Schulz iteration \cite{bjorck1971iterative,higham2008functions} underlying Muon to approximate the Schatten-$(p+1)$ norm update direction $U \Sigma^{1/p} V^T$. Starting from $X_0 = G / \|G\|_2$, the odd-degree update

\begin{equation*}
    X_{t+1} = a_t X_t + b_t\, X_t X_t^T X_t + c_t\, (X_t X_t^T)^2 X_t
\end{equation*}

decouples into a purely scalar recursion on the singular values: $\sigma_{t+1} = a_t \sigma_t + b_t \sigma_t^3 + c_t \sigma_t^5$. PolarExpress~\cite{amsel2025polar} and \citet{grishina2025accelerating} leverage the attractive fixed point of $f(x) = 1$ to derive Chebyshev-optimal coefficients for the orthogonalization target. No such contractive structure exists for $f(x) = x^{1/p}$.
Therefore, we decompose the exponent as $q = p^{-1/T}$, so that each of the $T$ iterations independently targets the mild perturbation $x \mapsto x^q$. Since $q \to 1$ as $T$ grows, a degree-5 polynomial can match this target accurately on each layer. We fit the coefficients by solving the minimax problem

\begin{equation}
    \min_{a_t, b_t, c_t}\ \max_{x \in [\varepsilon_t, 1]}\ \bigl| a_t x + b_t x^3 + c_t x^5 - x^{q} \bigr|
    \quad \text{s.t.} \quad a_t + b_t + c_t = 1,
    \label{eq:linf_remez}
\end{equation}

via SLSQP over a Chebyshev-node discretization. The pinning constraint prevents singular-value blow-up at $x = 1$, and the domain for iteration $t+1$ is updated to be the image of the fitted polynomial on $[\varepsilon_t, 1]$.
This yields the fastest routine we consider in the benchmark of~\Cref{fig:matrix_approximation}, effectively matching the efficiency of Muon's Newton-Schulz algorithm.

\paragraph{Concurrent alternatives.}
Two recent methods target the same primitive but do not fit our adaptive setting. \citet{qi2026delving} run two coupled Newton-Schulz recursions to compute $A^{1/2}$ and $A^{-1/2}$ jointly, then recover $U\Sigma^{1/2}V^T = X A^{-1/2} \cdot A^{1/2}$. The iteration is accurate, but it doubles working memory and only supports dyadic exponents $p = 2^k$. A second line uses the Zolotarev integral $x^{-\alpha} = \tfrac{\sin(\alpha\pi)}{\pi}\int_0^\infty (t + x)^{-1}\, t^{-\alpha}\, dt$, discretized by Gauss-Legendre to reduce the fractional factor to $k$ shifted resolvents $(G^T G + s_j I)^{-1}$. We solve all $k$ in parallel with a single batched Cholesky on the $m \times m$ normal equations, using power iteration for the spectral-norm rescaling to avoid any SVD. This supports arbitrary $p$, but the batched factorization costs $\mathcal{O}(k m^3)$ time and lacks the tensor-core throughput of matmul-only Newton-Schulz, which places it on the slow end of the Pareto frontier (\Cref{fig:matrix_approximation}, green).

\paragraph{N.B.} A contemporary blog post studies these fractional maps in~\cite{srivastava2026}, although slower spectral normalization methods are used for better performance.

\section{Selecting a Maximum p Threshold}
\label{app:maximal_p}

In practice, establishing a bounded search interval for the optimal Schatten exponent $p^{*}$ is necessary to maintain both computational efficiency and numerical stability. We restrict our scalar line search to a maximum value of $p_{max} = 50$ based on two primary structural and algorithmic limitations:

\begin{itemize}
    \item \textbf{Objective Function Saturation:} As $p$ grows large, the singular value terms $\sigma^{1/p}$ and $\sigma^{2/p}$ approach unity. Consequently, the proxy optimality objective function $J(p)$ becomes essentially flat, as the correlation-to-curvature ratio reaches its saturation limit. At this stage, the computed update $U\Sigma^{1/50}V^{T}$ is practically indistinguishable from the exact Muon update $UV^{T}$, making further optimization along $p$ redundant.
    \item \textbf{Estimator Precision and Numerical Degradation:} The precision of our hardware-efficient estimator for the fractional polar factor $U\Sigma^{1/p}V^{T}$ inherently worsens for highly elevated values of $p$. As demonstrated in~\Cref{fig:maximal-p}, direct polynomial approximations (such as Newton-Schulz with Remez coefficients) experience severe lift-off near the origin and fail to maintain acceptable error bounds when $p \ge 10$. While our primary Taylor approximation built atop the polar factor degrades much more gracefully and remains stable at higher exponents, its relative Frobenius error still increases as $p \to \infty$. Capping $p_{max}$ at 50 ensures that the optimizer relies on the Taylor approximation only where it remains demonstrably accurate.
\end{itemize}

Below, in~\Cref{fig:maximal-p}, we illustrate the approximation of the $x^{1/p}$ functional using different coefficient based methods with an iterative budget of 5; along with this, we benchmark the performance of methods in terms of mean relative Frobenius error using $25$ randomly drawn $1024 \times 1024$ Gaussian matrices. 

\begin{figure}[h!]
    \centering
    \includegraphics[width=0.7\linewidth]{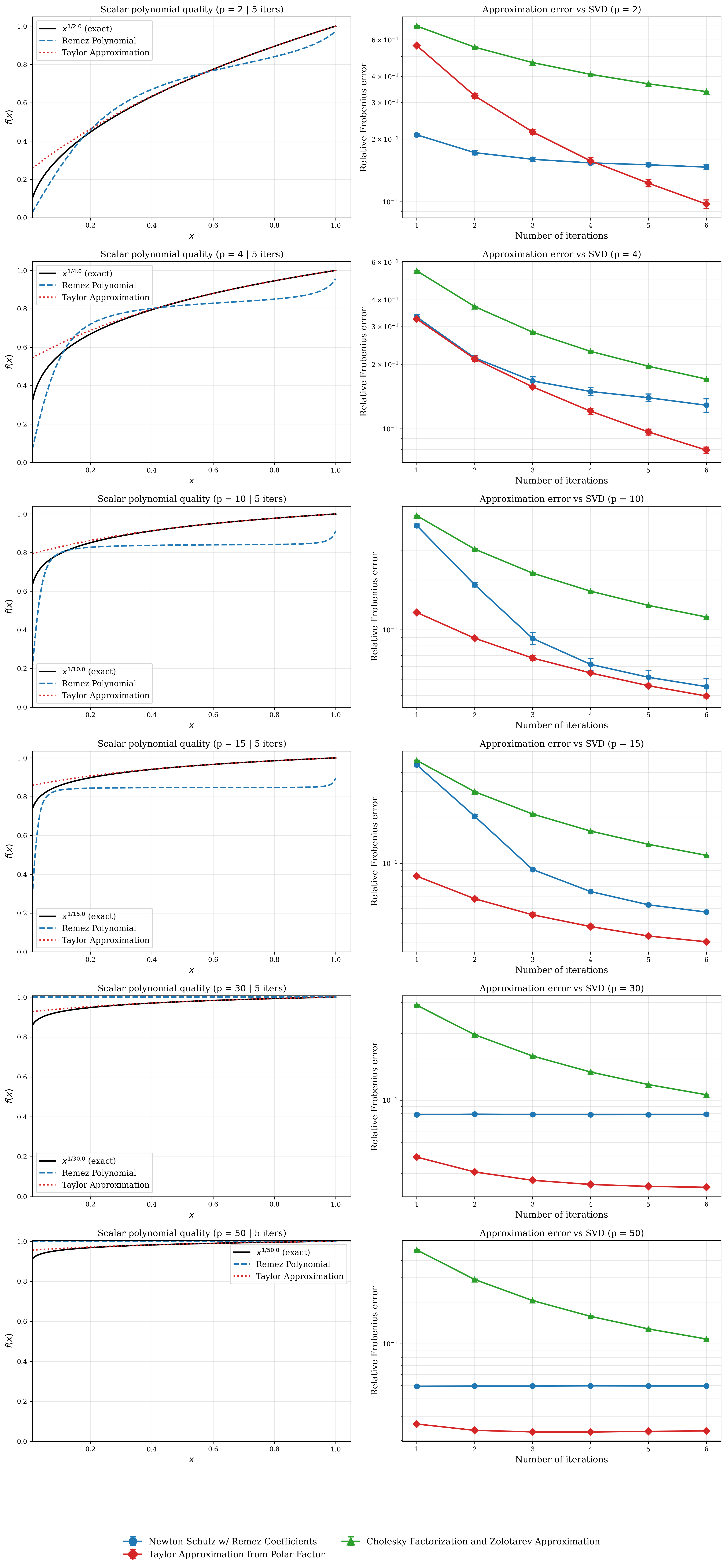}
    \caption{We plot the approximation of quality of $U \Sigma^{1/p} V^T$ for the different methods we propose.}
    \label{fig:maximal-p}
\end{figure}